\documentclass[11pt]{article}

\usepackage{subfig}
\usepackage{tikz}
\usepackage{dsfont}

\def\colorful{1}

\usepackage{amsthm,amsfonts,amsmath,amssymb,epsfig,color,float,graphicx,verbatim, enumitem}
\usepackage{multirow}
\usepackage{algorithm}
\usepackage[noend]{algpseudocode}
\usepackage{bm}
\newif\ifhyper\IfFileExists{hyperref.sty}{\hypertrue}{\hyperfalse}
\hypertrue
\ifhyper\usepackage{hyperref}\fi

\usepackage{enumitem}

\def\nnewcolor{0}
\ifnum\nnewcolor=1
\newcommand{\nnew}[1]{{\color{red} #1}}
\fi
\ifnum\nnewcolor=0
\newcommand{\nnew}[1]{#1}
\fi

\ifnum\colorful=1

\else

\fi
\newcommand{\mcal}[1]{\mathcal{#1}}
\newcommand{\var}{\mathbf{Var}}

\newtheorem{theorem}{Theorem}[section]

\newtheorem{lemma}[theorem]{Lemma}
\newtheorem{informal theorem}[theorem]{Theorem (informal statement)}

\newtheorem{proposition}[theorem]{Proposition}

\newtheorem{fact}[theorem]{Fact}

\newtheorem{remark}[theorem]{Remark}

\theoremstyle{definition}
\newtheorem{definition}[theorem]{Definition}
\newcommand{\eqdef}{\stackrel{{\mathrm {\footnotesize def}}}{=}}

\newcommand{\relu}{\phi}
\DeclareMathOperator{\rank}{rank}
\newcommand{\iu}{{i\mkern1mu}}

\newcommand{\lp}{\left}
\newcommand{\rp}{\right}

\newcommand\snorm[2]{\left\| #2 \right\|_{#1}}
\renewcommand\vec[1]{\mathbf{#1}}
\DeclareMathOperator*{\Prob}{\mathbf{Pr}}
\DeclareMathOperator*{\E}{\mathbf{E}}
\newcommand{\proj}{\mathrm{proj}}

\def\d{\mathrm{d}}
\newcommand{\normal}{\mathcal{N}}

\DeclareMathOperator*{\argmin}{argmin}

\newcommand{\bx}{\mathbf{x}}
\newcommand{\by}{\mathbf{y}}
\newcommand{\bv}{\mathbf{v}}
\newcommand{\bw}{\mathbf{w}}

\newcommand{\Sp}{\mathbb{S}}

\newcommand{\x}{\mathbf{x}}

\newcommand{\R}{\mathbb{R}}

\newcommand{\Z}{\mathbb{Z}}
\newcommand{\N}{\mathbb{N}}
\newcommand{\eps}{\epsilon}

\newcommand{\pr}{\mathbf{Pr}}
\newcommand{\poly}{\mathrm{poly}}

\newcommand{\D}{\mathcal{D}}

\newcommand\matr[1]{\bm{#1}}

\newcommand{\littlesum}{\mathop{\textstyle \sum}}

\newcommand{\wt}{\widetilde}

\usetikzlibrary{angles, quotes}
\usepackage{pgfplots}

\newcommand{\dotp}[2]{\left\langle #1, #2 \right\rangle}

\title{Algorithms and SQ Lower Bounds for PAC Learning
	One-Hidden-Layer ReLU Networks}

\author{
Ilias Diakonikolas\thanks{Supported by NSF Award CCF-1652862 (CAREER), 
a Sloan Research Fellowship, and a DARPA Learning with Less Labels (LwLL) grant.}\\
University of Wisconsin-Madison\\
{\tt ilias@cs.wisc.edu}\\
\and
Daniel M. Kane\thanks{Supported by NSF Award CCF-1553288 (CAREER) and a Sloan Research Fellowship.}\\ University of California, San Diego
\\
{\tt dakane@cs.ucsd.edu}
\and
Vasilis Kontonis\\
University of Wisconsin-Madison\\
{\tt kontonis@wisc.edu }\\
\and
Nikos Zarifis\thanks{Supported in part by a DARPA  Learning with Less Labels (LwLL) grant.}\\
University of Wisconsin-Madison\\
{\tt zarifis@wisc.edu}\\
}

\begin{document}

\maketitle

\begin{abstract}
We study the problem of PAC learning one-hidden-layer ReLU networks
with $k$ hidden units
on $\R^d$ under Gaussian marginals in the presence of additive label noise. 
For the case of positive coefficients, we give the first polynomial-time algorithm 
for this learning problem for $k$ up to $\tilde{O}(\sqrt{\log d})$. 
Previously, no polynomial time algorithm was known, even for $k=3$.
This answers an open question posed by~\cite{Kliv17}. Importantly,
our algorithm does not require any assumptions about the rank of the weight matrix
and its complexity is independent of its condition number. On the negative side,
for the more general task of PAC learning one-hidden-layer ReLU networks with arbitrary real coefficients, 
we prove a Statistical Query lower bound of $d^{\Omega(k)}$. Thus, we provide a 
separation between the two classes in terms of efficient learnability.
Our upper and lower bounds are general, extending to broader families of activation functions.
 
\end{abstract}

\setcounter{page}{0}
\thispagestyle{empty}
\newpage

\section{Introduction}

\subsection{Background and Motivation} \label{ssec:background}

In recent years, the impressive practical success of deep learning has motivated 
the development of provably efficient learning algorithms for various classes of neural networks.
A large body of research (see Section~\ref{ssec:related} for a brief overview) 
has resulted in efficient learning algorithms for shallow networks 
with common activation functions (e.g., ReLUs or sigmoids)
under various assumptions on the underlying distribution and the weight structure of the network.
Despite intensive investigation, the broad question of whether deep neural networks 
are efficiently learnable with provable guarantees remains an outstanding theoretical challenge 
in machine learning. In particular, the class of networks for which efficient learners 
are known is relatively limited, even in the realizable case 
(i.e., when the data is drawn from a neural network in the class).

In this work, we continue this line of investigation by studying the learnability of a simple class of networks
without imposing strong restrictions on the structure of its weights. 
Specifically, we focus on the problem of learning one-hidden-layer ReLU 
networks under the Gaussian distribution 
in the presence of additive random label noise.
Our goal is to understand the complexity 
of this problem {\em in the PAC learning model without assumptions 
on the weight matrix of the network}. 

\begin{definition}[One-hidden-layer ReLU networks] \label{def:sum-of-relus}
Let $\mathcal{C}_k$ denote the concept class of one-hidden-layer ReLU networks
on $\R^d$ with $k$ hidden units. That is, $f_{\alpha, \mathbf{W}} \in \mathcal{C}_k$ if and only if 
there exist weight vectors $\bw^{(i)} \in \R^d$ and real coefficients $\alpha_i$, $i \in [k]$, 
such that $f_{\alpha, \mathbf{W}}(\bx) = \sum_{i=1}^k \alpha_i \relu (\langle \bw^{(i)}, \bx \rangle)$, 
where $\relu(t) = \max\{0, t\}$, $t \in \R$.
We will denote by $\alpha = (\alpha_i)_{i=1}^k$ the vector of coefficients and by $\mathbf{W} = [\bw^{(i)}]_{i=1}^k$
the weight matrix of the network.
We will use $\mathcal{C}_k^{+}$ to denote the subclass of $\mathcal{C}_k$ where $\alpha \in \R_+^{k}$.
\end{definition}

The (distribution-specific) PAC learning problem for a concept class $\mathcal{C}$ of real-valued functions 
is the following: 
The input is a multiset of i.i.d. labeled examples $(\bx, y)$, where $\bx$ is generated
from the standard Gaussian distribution on $\R^d$ and $y = f(\bx)+\xi$, where $f \in \mathcal{C}$ 
is the unknown target concept and $\xi$ is some type of random observation noise. 
The goal of the learner is to output a hypothesis
$h: \R^d \to \R$ that with high probability is close to $f$ in $L_2$-norm. 
The hypothesis $h$ is allowed to lie in any efficiently representable hypothesis class $\mathcal{H}$.
If $\mathcal{H} = \mathcal{C}$, the PAC learning algorithm is called {\em proper}.

Perhaps surprisingly, the complexity of PAC learning one-hidden-layer ReLU networks (even with positive weights)
has remained open, even in the realizable setting, under Gaussian marginals, 
and for $k=3$~\cite{Kliv17}\footnote{Formally speaking, the $k=2$ case does not appear explicitly in the literature,
but an efficient algorithm easily follows from prior work on parameter estimation (e.g.,~\cite{GeLM18}).}. 
A line of prior work~\cite{GeLM18, BakshiJW19, GeKLW19} had studied the task of {\em parameter estimation}
for this concept class, i.e., the task of recovering the unknown coefficients $\alpha_i$
and weight vectors $\bw^{(i)}$ of the data generating network within small accuracy.
It should be noted that for parameter estimation to even be information-theoretically possible, 
some assumptions on the target function are necessary. 
The aforementioned prior works made the common assumption that the weight matrix 
$\mathbf{W} = [\bw^{(i)}]_{i=1}^k$ is {\em full-rank}. Under this assumption, they
provided efficient parameter learning algorithms 
with respect Gaussian marginals for the case of {\em positive coefficients}, i.e., for $\mathcal{C}_k^{+}$.
Importantly, the sample and computational complexity of these algorithms scale polynomially 
with the condition number of $\mathbf{W}$. In contrast, no such algorithm is known for 
general coefficients, i.e., for $\mathcal{C}_k$, even under the aforementioned strong assumptions
on the weights.

In contrast to parameter estimation, PAC learning one-hidden-layer ReLU networks 
does not require any assumptions on the structure of the weight matrix. The PAC learning 
problem for this class is information-theoretically solvable with polynomially
many samples. The question is whether a computationally efficient algorithm exists.
It should also be noted that proper PAC learning is not generally equivalent to parameter estimation,
as it is in principle possible to have two networks that define close-by functions and 
whose parameters are significantly different.

\subsection{Our Results} \label{ssec:results}
We are ready to describe the main contributions of this work.
Our main positive result is the first PAC learning algorithm for $\mathcal{C}_k^{+}$ (one-hidden-layer 
Relu networks with positive coefficients) under Gaussian marginals 
that runs in polynomial time for any $k = \tilde{O}(\sqrt{\log d})$. 
On the lower bound side, we establish a Statistical Query (SQ) lower bound suggesting 
that no such algorithm is possible for $\mathcal{C}_k$ (general coefficients) for any $k = \omega(1)$ 
(also under Gaussian marginals).
Our SQ lower bound provides a separation between $\mathcal{C}_k^{+}$ and $\mathcal{C}_k$ 
in terms of efficient learnability.

Before we state our main theorems, we formally define the PAC learning problem.

\begin{definition}[Distribution-Specific PAC Learning] \label{def:PAC}
Let $\mathcal{F}$ be a concept class of real-valued functions over $\R^d$, 
$\D$ be a distribution on $\R^d$, $\mathcal{F} \in L_2(\D, \R^d)$, and $0< \eps <1$.
Let $f$ be an unknown target function in $\mathcal{F}$. 
A {\em noisy example oracle}, $\mathrm{EX}^{\mathrm{noise}}(f, \mathcal{F})$,
works as follows: Each time $\mathrm{EX}^{\mathrm{noise}}(f, \mathcal{F})$ is invoked,
it returns a labeled example $(\bx, y)$, such that: (a) $\bx \sim \D$, 
and (b) $y = f(\bx) + \xi$, where $\xi$ is a zero-mean and standard deviation $\sigma$ 
subgaussian random variable that is independent of $\bx$. A learning algorithm is given i.i.d. samples from the noisy oracle
and its goal is to output a hypothesis $h$ such that with high probability $h$ is $\eps$-close to 
$f$ in $L_2$-norm, i.e., it holds $\E_{\bx \sim \D}[(f(\bx) - h(\bx))^2] \leq \eps^2 \left( \E_{\bx \sim \D}[f^2(\bx)] +\sigma^2\right)$.
\end{definition}

Our main positive result is the first computationally efficient PAC learning algorithm for $\mathcal{C}_k^{+}$.

\begin{theorem}[Proper PAC Learner for $\mathcal{C}_k^{+}$] \label{thm:alg-inf}
There is a proper PAC learning algorithm for $\mathcal{C}_k^{+}$ with respect to the standard Gaussian distribution on $\R^d$
with the following performance guarantee: The algorithm draws $\poly(k/\eps) \cdot \tilde{O}(d) $ noisy labeled examples 
from an unknown target $f \in \mathcal{C}_k^{+}$, runs in time $\poly(d/\eps)+ (k/\eps)^{O(k^2)}$, 
and outputs a hypothesis $h \in \mathcal{C}_k^{+}$ that with high probability is $\eps$-close to $f$ in $L_2$-norm.
\end{theorem}

\noindent Theorem~\ref{thm:alg-inf} gives the first polynomial-time PAC learning 
algorithm for one-hidden-layer ReLU networks
under any natural distributional assumptions, answering a question posed by~\cite{Kliv17}.
Our algorithm runs in polynomial time for some $k = \tilde{\Omega}(\sqrt{\log d})$. 
The existence of such an algorithm was previously open, even for $k=3$.

We remark that our main algorithmic result is more general, in the sense that it immediately 
extends to positive coefficient one-hidden-layer networks composed of any non-negative Lipschitz 
activation function. See Theorem~\ref{thm:alg} for a detailed statement.

Some additional remarks are in order: 
As stated in Theorem~\ref{thm:alg-inf}, 
our learning algorithm is proper, i.e., $h \in \mathcal{C}_k^{+}$.
An important distinguishing feature of our algorithm from prior related work is that it requires 
no assumptions on the weight matrix of the network, and in particular 
that its sample complexity is independent of its condition number. 
Prior work had given parameter estimation algorithms for this concept class
with sample complexity (and running time) polynomial in the condition number. 
On the other hand, the running time of our algorithm scales with $\exp(k)$, 
while previous parameter estimation algorithms had $\poly(k)$ dependence. 
The existence of a $\poly(k)$ time PAC learning algorithm
remains an outstanding open question.
An additional advantage of our algorithm is that it also immediately extends
to the agnostic setting and in particular is robust to a small (dimension-independent) amount of 
adversarial $L_2$-error. 

The algorithm of Theorem~\ref{thm:alg-inf} crucially uses the assumption that the coefficients of the target
network are positive. A natural question is whether an algorithm with similar guarantees 
can be obtained for unrestricted coefficients. Perhaps surprisingly, we provide evidence that such an algorithm
does not exist. Specifically, our second main result is a correlational Statistical Query (SQ) 
lower bound ruling out a broad family of $\poly(d)$-time algorithms for $\mathcal{C}_k$ 
for $\eps  = \Omega(1)$, for {\em any} $k = \omega(1)$.

Specifically, we prove a lower bound for PAC learning $\mathcal{C}_k$
under Gaussian marginals in the correlational SQ model.
A correlational SQ algorithm has query access to the target concept
$f: \R^d \to \R$ via the following oracle: The oracle takes as input any 
bounded query function $q: \R^d \to [-1, 1]$ and an accuracy parameter $\tau>0$, and outputs
an estimate $\gamma$ of the expectation $\E_{\bx \sim \D}[f(\bx)q(\bx)]$ such that 
$|\gamma - \E_{\bx \sim \D}[f(\bx)q(\bx)]| \leq \tau.$ 
We note that the correlational SQ model captures a broad family of algorithms, including first-order methods 
(e.g., gradient-descent), dimension-reduction, and moment-based methods. 
(In particular, our algorithm establishing Theorem~\ref{thm:alg-inf} can be easily simulated in this model.)
We establish the following:

\begin{theorem}[Correlational SQ Lower Bound for $\mathcal{C}_k$] \label{thm:SQ-lb-inf}
Any correlational SQ learning algorithm for $\mathcal{C}_k$ under the standard Gaussian 
distribution on $\R^d$ that guarantees error 
$\eps  = \Omega(1)$ requires either queries of accuracy $d^{-\Omega(k)}$ 
or $2^{d^{\Omega(1)}}$ many queries.
\end{theorem}

The natural interpretation of Theorem~\ref{thm:SQ-lb-inf} is the following:
If the SQ algorithm uses statistical queries of accuracy $d^{-\Omega(k)}$, 
then simulating a single query with iid samples would require $d^{\Omega(k)}$ samples (hence time). 
Otherwise, the algorithm would require $2^{d^{\Omega(1)}}$ time (since each query requires at least one unit of time).
Theorem~\ref{thm:SQ-lb-inf}, combined with our Theorem~\ref{thm:alg-inf}, provides a (super-polynomial) 
computational separation between the PAC learnability of $\mathcal{C}_k$ 
and $\mathcal{C}_k^{+}$ in the correlational SQ model.

We note that the statement of our general SQ lower bound (Theorem~\ref{thm:sq_theorem}) is much more general
than Theorem~\ref{thm:SQ-lb-inf}. Specifically, we obtain a correlational
SQ lower bound for PAC learning (under Gaussian marginals) 
a class of functions of the form $\sigma(\sum_{i=1}^k  \alpha_i \phi(\bw^{(i)}, \bx))$,
where roughly speaking $\sigma$ is any odd non-vanishing function and $\phi$ is not a low-degree polynomial.

\subsection{Our Techniques} \label{ssec:techniques}
Here we provide an overview of our techniques
in tandem with a comparison to prior work.
We start with our algorithm establishing Theorem~\ref{thm:alg-inf}.
Our learning algorithm for $\mathcal{C}_k^{+}$ employs a data-dependent
dimension reduction procedure. Specifically, we give an efficient method
to reduce our $d$-dimensional learning problem down to a $k$-dimensional
problem, that can in turn be efficiently solved by a simple covering method.

Let $f(\bx) = \sum_{i=1}^k \alpha_i \relu (\langle \bw^{(i)}, \bx \rangle)$ be the 
target function and observe that $f$ depends only on the $k$ unknown linear forms 
$\langle \bw^{(i)}, \bx \rangle$, $i \in [k]$. If we could identify the subspace $V$ 
spanned by the $\bw^{(i)}$'s exactly, then we could also identify $f$ by brute-force 
on $V$, noting that we only need to search a $k^2$-dimensional space of functions
and  that for any $\bx \in \R^d$ it holds $f(\bx) = f(\proj_V(\bx))$. Our algorithm is 
based on a robust version of this idea. In particular, if we can find a subspace $V'$ 
that closely approximates $V$, then it suffices to solve for $f$ on $V'$ and 
use this projection to obtain an approximation to $f$. 

To find a subspace $V'$ approximating $V$, 
we consider the matrix of degree-$2$ Chow parameters (second moments) 
of $f$, i.e., $\E_{\bx \sim \mathcal{N}(0, I)} [f(\bx) (\bx\bx^T-\mathbf{I})]$. 
It is not hard to see that the (normalized) second moments of $f$ are positive 
in the directions along $V$ and $0$ in orthogonal directions. 
Thus, if we could compute the second moments exactly, we could solve for $V$ 
as the span of the second moment matrix. Unfortunately, we can only approximate 
the true second moment matrix via samples. To deal with this approximation, we note that 
the true second moments will be large in the direction of $\bw^{(i)}$ 
for components with large coefficients $\alpha_i$ and $0$ in directions orthogonal to $V$. 
Using this fact, we show that if $V'$ is the span of the $k$ largest eigenvalues
of an approximate second moment matrix (obtained via sampling), the weight vectors $\bw^{(i)}$ 
corresponding to the important components of $f$ will still be close to $V'$. From this point,
can use a net-based argument to find a hypothesis $h \in \mathcal{C}_k^{+}$ with weight vectors 
on $V'$ so that $f(\bx)$ is close to $h(\proj_{V'}(\bx))$ in $L_2$-norm.

We note that the idea of using dimension-reduction to find a low-dimensional invariant subspace 
has been previously used in the context of PAC learning intersections of LTFs~\cite{Vempala10a, DKS18-nasty}. 
Our algorithm and its analysis of correctness are quite different from these prior works.
We also note that \cite{GeLM18} also used information based on low-degree moments 
for their parameter estimation algorithm, but in a qualitatively different way. In particular, \cite{GeLM18}
used tensor-decomposition techniques (based on moments of degree up to four) to uniquely identify
the weight vectors, under structural assumptions on the weight matrix (full-rank and bounded condition number).

We now proceed to explain our SQ lower bound construction.
As is well-known, there is a general methodology to establish such lower bounds,
via an appropriate notion of SQ dimension~\cite{BFJ+:94, FeldmanGRVX17}.
In our setting, to prove an SQ lower bound, 
it suffices to find a large collection of functions $f_1,\ldots,f_m \in \mathcal{C}_k$ 
with the following properties: (1) The $f_i$'s are pairwise far away from each other, and (2) 
The $f_i$'s have small pairwise correlations. The difficulty is, of course, to construct such a family.
We describe our construction in the following paragraph.

First, it is not hard to see that (1) and (2) can only be simultaneously satisfied 
if almost all of the $f_i$'s have nearly-matching low-degree moments. 
In fact, we provide a construction in which all the low-degree moments of
all of the $f_i$'s vanish. To achieve this, we build on an idea introduced in~\cite{DKS17-sq}. 
Roughly speaking, the idea is to define a family of functions whose interesting information
is hidden in a random low-dimensional subspace, so that learning an unknown function in the family 
amounts to finding the hidden subspace. In more detail, we will define 
a function in two dimensions which has the correct moments, and then
embed it in a randomly chosen subspace. 

For simplicity, we explain our $2$-dimensional construction for ReLU activations, 
even though our SQ lower bound is more general.
We provide an explicit $2$-dimensional construction of a mixture 
$F$ of $2k$ ReLUs whose first $k-1$ moments vanish exactly. 
For any $2$-dimensional subspace $V$, we can define $F_V(\bx) = F(\proj_V(\bx)).$ 
From there, we can show that if $U$ and $V$ are two subspaces that are far apart 
--- in the sense that no unit vector in $U$ has large projection in $V$ ---  then $F_U$ and $F_V$ 
will have small correlation --- on the order of the $k$-th power of the closeness parameter 
between the defining subspaces. Moreover, it is not hard to show that two 
randomly chosen $U$ and $V$ are far from each other with high probability. 
This allows us to find an exponentially large family of $F_V$'s that have pairwise
exponentially small correlation.

\subsection{Related Work} \label{ssec:related}
In recent years, there has been an explosion of research
on provable algorithms for learning neural networks
in various settings, see, e.g.,~\cite{Janz15, SedghiJA16, DanielyFS16, ZhangLJ16, 
ZhongS0BD17, GeLM18, GeKLW19, BakshiJW19, GoelKKT17, Manurangsi18, GoelK19, VempalaW19} for 
some works on the topic. 
The majority of these works focused on parameter learning, i.e., the problem of 
recovering the weight matrix of the data generating neural network. 
In contrast, the focus of this paper is on PAC learning. 
We also note that PAC learning of simple classes of neural networks has been studied 
in a number of recent works~\cite{GoelKKT17, Manurangsi18, GoelK19, VempalaW19}. 
However, the problem of PAC learning linear combinations of (even) $3$ ReLUs 
under any natural distributional assumptions (and in particular under the Gaussian distribution) 
has remained open. At a high-level, prior works either rely on tensor 
decompositions~\cite{SedghiJA16, ZhongS0BD17, GeLM18, GeKLW19, BakshiJW19} 
or on kernel methods~\cite{ZhangLJ16, DanielyFS16, GoelKKT17, GoelK19}.
In the following paragraphs, we describe in detail the prior works
more closely related to the results of this paper.

The work of~\cite{GeLM18} studies the parameter learning of positive linear combinations of ReLUs
under the Gaussian distribution in the presence of additive (mean zero sub-gaussian) noise.
That is, they consider the same concept class and noise model as we do, but study parameter learning
as opposed to PAC learning. \cite{GeLM18} show that the parameters can be approximately
recovered efficiently, under the assumption that the weight matrix is full-rank with bounded condition number.
The sample complexity and running time of their algorithm scales polynomially with the condition number.
More recently,~\cite{BakshiJW19, GeKLW19} obtained efficient parameter learning algorithms for vector-valued 
depth-$2$ ReLU networks under the Gaussian distribution. Similarly, the algorithms in these works 
have sample complexity and running time scaling polynomially with the condition number. 
We note that the algorithmic results in the aforementioned works do not apply 
to $\mathcal{C}_k$, i.e., the class of arbitrary linear combinations of ReLUs. 

\cite{VempalaW19} show that gradient descent agnostically PAC learns low-degree polynomials 
using neural networks as the hypothesis class. 
Their approach has implications for (realizable) PAC learning of certain neural networks 
under the uniform distribution on the sphere. We note that their method implies an algorithm 
with sample complexity and running time exponential in $1/\eps$, even for a single ReLU.
\cite{GoelK19} give an efficient PAC learning algorithm for certain $2$-hidden-layer neural networks
under arbitrary distributions on the unit ball. We emphasize that their algorithm does not apply 
for (positive) linear combinations of ReLUs. In fact, recent work 
has shown that the problem we solve in this paper is NP-hard under arbitrary distributions, 
even for $k=2$~\cite{GKMR20}.

The SQ model was introduced by~\cite{Kearns:98} in the context of learning Boolean-valued functions
as a natural restriction of the PAC model~\cite{Valiant:84}. 
A recent line of work~\cite{Feldman13, FeldmanPV15, FeldmanGV15, Feldman16}
extended this framework to general search problems over distributions. 
One can prove unconditional lower bounds on the computational complexity 
of SQ algorithms via an appropriate notion of {\em Statistical Query dimension}.
A lower bound on the SQ dimension of a learning problem provides an unconditional 
lower bound on the computational complexity of any SQ algorithm for the problem.

The work of \cite{VempalaW19} establishes correlational SQ lower bounds for learning a class
of degree-$k$ polynomials in $d$ variables.  \cite{Shamir18} shows that gradient-based algorithms 
(a special case of correlational SQ algorithms) cannot efficient learn certain families
of neural networks under well-behaved distributions (including the Gaussian distribution). 
We note that the lower bound constructions in these works do not imply corresponding lower bounds 
for one-hidden-layer ReLU networks.

\paragraph{Concurrent and Independent Work.} Contemporaneous work~\cite{GGJKK20}, using a different construction, obtained super-polynomial SQ lower bounds for learning one-hidden-layer 
neural networks (with ReLU and other activations) under the Gaussian distribution. 
 \section{Preliminaries} \label{sec:prelims}

\noindent {\bf Notation.}
For $n \in \Z_+$, we denote $[n] \eqdef \{1, \ldots, n\}$.  We will use small
boldface characters for vectors.  For $\bx \in \R^d$, and $i \in [d]$,
$\bx_i$ denotes the $i$-th coordinate of $\bx$, and $\|\bx\|_2 \eqdef
(\littlesum_{i=1}^d \bx_i^2)^{1/2}$ denotes the $\ell_2$-norm of $\bx$.
We denote by $\snorm{2}{\matr A}$ the spectral norm of matrix $\matr A$.
We will use $\langle \bx, \by \rangle$ for the inner product between $\bx, \by
\in \R^d$.  We will use $\E[X]$ for the expectation of random variable $X$
and $\pr[\mathcal{E}]$ for the probability of event $\mathcal{E}$.  We denote
by $\var[X]$ its variance.

For $d\in \mathbb{N}$, we denote
$\Sp^{d-1}$ the $d$-dimensional sphere.  Denote by $ \theta(\vec u, \vec
v)$ the angle between the vectors $\vec u,\vec v$.
For a vector of weights $\vec \alpha = (\alpha_1,\ldots, \alpha_k) \in \R^{2k}$,
and matrix $\matr W \in \R^{k \times d}$ we denote
$ f_{\alpha, \matr W}(\vec x) = \alpha^T \phi(\matr W x) =\sum_{i=1}^{k} \alpha_i\ \phi(\langle \vec w^{(i)}, \vec x \rangle)$.
Let $\mathcal{N}$ denote the standard univariate Gaussian distribution, we also denote $\mathcal{N}^2$ the two dimensional Gaussian distribution and  $\mathcal{N}^d$ the $d$-dimensional one.
 \section{Efficient Learning Algorithm} \label{sec:alg}
In this section, we give our upper bound for the problem of learning positive linear combinations
of  Lipschitz activations, thereby establishing Theorem~\ref{thm:alg-inf}. We prove the
following more general statement:

\begin{theorem}[Learning Sums of Lipschitz Activations]
  \label{thm:alg}
 Let $f(\vec x) = \sum_{i=1}^k \alpha_i \relu\big(\dotp{\vec w^{(i)}}{\vec x}\big)$ with $\alpha_i>0$ for all $i\in[k]$,
 where $\phi(t)$ is an $L$-Lipschitz, non-negative activation function
  such that $\E_{t \sim \normal}[\phi(t)] \geq C$, $\E_{t\sim \normal}[\relu(t)(t^2-1)]\geq C$, where $C>0$ and $\E_{t \sim \normal}[\phi^2(t)]$ is finite. There
  exists an algorithm that given $k\in \mathbb{N}, \eps>0$, and sample access
  to a noisy set of samples from $f: \R^d \rightarrow \R_+$, draws $m =
  d \cdot \poly(k, 1/\eps)\cdot \poly(L/C)$ samples, runs in time $\poly(m) + \wt{O}((1/\eps)^{k^2})$,
  and outputs a proper hypothesis $h$ that, with probability at least $9/10$,
  satisfies \[ \E_{\vec x \sim\normal^d}[(f(\vec x) - h(\vec x))^2]
  \leq \eps^2 \poly(L/C)
\left(\sigma^2 + \E_{\vec x \sim \normal^d}[f(\x)^2]\right)\;.
\]
\end{theorem}
\begin{remark}
Theorem~\ref{thm:alg-inf} follows as a corollary of the above, by noting
that the ReLU satisfies $L=1$ and $C=\frac{1}{\sqrt{2\pi}}$.
\end{remark}

The following fact gives formulas for the low-degree Chow parameters
of a one-layer network (see Appendix~\ref{app:upper_bound}, Fact~\ref{fct:chow_app}).
\begin{fact}[Low-degree Chow Parameters] \label{lem:chow_formulas}
Let $f: \R^d \to \R$ be of the form $f(\vec x) = \sum_{i=1}^k \alpha_i\cdot \relu\lp(\dotp{\vec w^{(i)}}{\vec x}\rp)$.
Then
$
\E_{\vec x \sim \normal^d}\lp[f(\vec x)\right]
=\E_{t\sim \normal}[\relu(t)] \sum_{i=1}^k \alpha_i
\;,$
$  \E_{\vec x \sim \normal^d}\lp[f(\vec x) \vec x\right]
    =\E_{t\sim \normal}[\relu(t)t]\cdot \sum_{i=1}^k \alpha_i \vec w^{(i)} \;,$
and
\begin{align} \label{eq:degree_2_formula}
\matr A = \E_{\vec x \sim \normal^d}\lp[f(\vec x) ( \vec x\vec x^T-\matr I)\right]=
\E_{t\sim \normal}[\relu(t)(t^2-1)] \sum_{i=1}^k \alpha_i\vec w^{(i)}{\vec w^{(i)}}^T \;.
\end{align}
\end{fact}
The crucial formula is the one of the degree-$2$ Chow parameters,
Equation~\eqref{eq:degree_2_formula}. In fact, we can already describe the main
idea of our upper bound.  Let us assume that we have the degree-$2$
Chow parameters matrix $\matr A$ {\em exactly}. Then, by using singular value
decomposition, we would obtain a basis of the vector space spanned by the
parameters $\vec w^{(i)}$.  The dimension of this space is at most $k$ and
therefore in that way we essentially reduce the dimension of the problem from
$d$ down to $k$.  To find parameters $\hat{\alpha}_i, \hat{\vec w}^{(i)}$
that give small mean squared error, we can now make a grid $\mcal G$ and pick
the ones that minimize the empirical mean squared error with the samples,
that is
$$
\min_{\vec \beta, \matr U \in \mcal G}
\sum_{i=1}^m (f_{\vec \beta, \matr U}(\vec x^{(i)}) - y^{(i)})^2 \;.
$$
Even though we do not have access to the matrix $\matr A$ exactly, we can estimate
it empirically. Since the activation function $\relu(\cdot)$ is well-behaved
and the distribution of the examples is Gaussian, we can
get a very accurate estimate of $\matr A$ with roughly $\wt{O}(d k /\eps^2)$
samples.  We give the following lemma whose proof relies on matrix
concentration and concentration of polynomials of Gaussian random variables
(see Appendix~\ref{app:chow_estimation}, Lemma~\ref{lem:appendix_chow_est}).

\begin{lemma}[Estimation of degree-$2$ Chow parameters]
  \label{lem:empirical_chow_parameters}
   Let $ f_{\vec \alpha, \matr W}(\vec x) = \sum_{i=1}^k
 \alpha_i \phi(\langle \vec w^{(i)}, \vec x \rangle)$,
 where $\phi(t)$ is an $L$-Lipschitz, non-negative activation function
 such that $\E_{t \sim \normal}[\phi(t)] \geq C$. Let $\matr \Sigma = \E_{\vec x \sim \normal^d}[ f_{\alpha, \matr W}(\vec x) \vec x \otimes \vec x ]$
 be the degree-$2$ Chow parameters of $f_{\alpha, \matr W}$.
 Then, for some $N = \wt{O}(d k/\eps^2)$ samples $(\vec x^{(i)},
 y^{(i)})$, where $y^{(i)}=f_{\vec \alpha, \matr W}(\vec x^{(i)}) +\xi_i$ and
 $\xi_i$ is a zero-mean, subgaussian noise with variance $\sigma^2$,
 it holds with probability at least $99\%$ that
 $$\snorm{2}{ \frac{1}{N} \sum_{i=1}^N \vec x^{(i)} \otimes \vec x^{(i)} y^{(i)} - \matr \Sigma }
 \leq \eps \left(
 \sigma + \frac{L}{C} \E_{\vec x \sim \normal^d}[f_{\vec \alpha, \matr W}(\vec x)]
 \right) \;.
 $$
\end{lemma}

The next step is to quantify how accurately we need to estimate the degree-$2$ Chow
parameters, so that doing SVD on the empirical matrix gives us a good
approximation of the subspace spanned by the true parameters $\vec w^{(i)}$.
We show that that estimating the degree-$2$ Chow parameter matrix
within spectral norm roughly $\eps/k$ suffices. In particular, we show that
the top-$k$ eigenvectors of our empirical estimate span
approximately the subspace where the true parameters $\vec w^{(i)}$ lie.
For the proof, we are going to use the following lemma
that bounds the difference of a function evaluated at
correlated normal random variables.
\begin{lemma}[Correlated Differences, Lemma 6 of \cite{KTZ19}]
  \label{lem:correlated_differences}
  Let $r(\vec x) \in L_2(\R^d, \normal^d)$ be differentiable almost
  everywhere and let
  \[
    D_\rho =
    \normal\lp(\vec 0,
    \begin{pmatrix}
      \matr I & \rho \matr I \\
      \rho \matr I & \matr I
    \end{pmatrix}
    \rp).
  \]
  We call $\rho$-correlated a pair of random variables $(\vec x, \vec y) \sim
  D_{\rho}$.  It holds
  \[
    \frac{1}{2}
    \E_{(\vec x, \vec z) \sim D_\rho}[(r(\vec x) - r(\vec z))^2]
    \leq (1-\rho) \E_{\vec x \sim \normal^d}\lp[ \snorm{2}{\nabla r(\vec x)}^2 \rp]\, .
  \]
\end{lemma}
We are now ready to prove the key technical lemma of our approach.
We remark that the following dimension reduction lemma is rather general
and holds for any reasonable activation function, in the sense that
the error is bounded as long as its expected derivative
$\E_{t \sim \normal}[(\phi'(t))^2]$ is bounded.
 \begin{lemma}[Dimension Reduction]\label{lem:dimension_reduction}
   Let $f_{\vec \alpha,\matr W}(\vec x) = \sum_{i=1}^k \alpha_i
   \phi\big(\dotp{\vec w^{(i)}}{\x}\big)$ with $\alpha_i>0$, let $\matr
   A=\E_{\vec x\sim \normal^d}[f(\vec x)\vec x\vec x^T ]$ and $ \E_{t\sim \normal}[\phi(t)(t^2-1)]= C_1$. Let $\matr M \in
   \R^{d \times d}$ be a matrix
   such that $\snorm{2}{\matr A - \matr M}^2 \leq \eps$
   and let $\cal V$ be the subspace of $\R^d$
   that is spanned by the top-$k$ eigenvectors of $\matr M$. There exist $k$ vectors
   $\vec v^{(i)}\in \cal V$ such that for the matrix $\matr V\in \R^{k\times d}$ constructed by the vectors $\vec v^{(i)}$, it holds
   $
   \E_{\vec x\sim \normal^d}[(f_{\vec \alpha,\matr W}(\vec x) -
   f_{\vec \alpha,\matr V}(\vec x))^2]
   \leq 2 k \eps
   \E_{\vec x\sim \normal^d}[f_{\vec \alpha,\matr W}(\vec x)]
   \E_{t\sim \normal}[(\phi'(t))^2]/C_1\;.
$
 \end{lemma}
 \begin{proof}
  For simplicity, let us denote $\eps = \snorm{2}{\matr A - \matr M}$.
  Moreover, let $\matr A'= \matr A- \E_{t\sim \normal}[\relu(t)] \sum_i\alpha_i \matr I$,
  $\matr M'=\matr M- \E_{t\sim \normal}[\relu(t)] \sum_i\alpha_i \matr I$, and
  observe that $\snorm{2}{\matr A-\matr M}=\snorm{2}{\matr A'-\matr M'}$. We note that $\matr A' =\E_{t\sim \normal}[\relu(t)(t^2-1)] \sum_{i=1}^k \alpha_i\vec w^{(i)}{\vec w^{(i)}}^T $, from Fact~\ref{lem:chow_formulas}.
  Let $\vec b^{(1)},\ldots, \vec b^{(k)}$ be the eigenvectors corresponding
  to the top-$k$ eigenvalues of $\matr M'$ (which are also the top $k$
  eigenvectors of $\matr M$), and let ${\cal V}=\mathrm{span}(\vec b_1,\ldots, \vec b_k)$.
  Let $\vec v^{(i)}=\proj_{\cal V}(\vec w^{(i)})$ and $\vec r^{(i)}=\vec w^{(i)}-\vec v^{(i)} $.
  Let $\vec v^{(1)}, \ldots, \vec v^{(k)}$ be any $k$ vectors in $\R^d$.
  Then we have,
  \begin{align}
    \label{eq:fourier_inequality}
    \E_{\vec x \sim \normal^d}
    [(f_{\vec \alpha, \matr W}(\vec x) - f_{\vec \alpha, \matr V}(\vec x))^2 ]
    &\leq k
    \sum_{i=1}^k \alpha_i^2 \E_{\vec x \sim \normal^d}
    \lp[
    \left(\phi\big(\langle{\vec w^{(i)}},{\vec x}\rangle\big) -
    \phi\big(\langle{\vec v^{(i)}},{\vec x\rangle}\big)\right)^2
    \rp]
    \nonumber
    \\
    &\leq 2 k
    \E_{t \sim \normal} \lp[ (\phi'(t))^2 \rp]
    \sum_{i=1}^k \alpha_i^2 (1- \langle\vec w^{(i)},\vec v^{(i)}\rangle),
  \end{align}
  where for the last inequality we used
  Lemma~\ref{lem:correlated_differences} and the fact that the random
  variables
  $\langle \vec w^{(i)}, \vec x \rangle$ and $\langle \vec v^{(i)}, \vec x
  \rangle$ are $\rho_i$-correlated with
  $\rho_i = \langle \vec w^{(i)}, \vec v^{(i)} \rangle$.

  It suffices to prove that $\snorm{2}{\vec r^{(i)}}=\snorm{2}{\vec
  w^{(i)}-\vec v^{(i)}}\leq \eps'$ for some sufficiently small $\eps'$.
  Note that because $\vec r^{(i)} \in \cal V^{\perp}$, it holds
  ${\vec r^{(i)}}^T \matr M' \vec r^{(i)} \leq \snorm{2}{\vec r^{(i)}}^2 \max_{\vec u \in
  {\cal V}^{\perp}} \frac{\vec u^T \matr M' \vec u}{\snorm{2}{\vec u}}$,
  because we know that the subspace $\cal W$ is spanned by the top $k$ eigenvectors
  of $\matr M'$. Let $\vec u=\sum_{i=1}^d \vec u^{(i)}$, where $\vec u^{(i)}\in \ker(\matr M-\lambda_i \matr I)$ for all
  $i\in\{k+1,\ldots, d\} $ and $\lambda_i$ is the $i$-th greatest
  eigenvalue. From Weyl's inequality, we have that if $A_i$ are the
  eigenvalues of $\matr A'$ in decreasing order then
  $\snorm{1}{A_i -\lambda_i}\leq \eps$ and we know that the eigenvalues of $\matr A'$ for $i>k$
  are zero, because the $\rank(\matr A)\leq k$.
  Thus,
  $$ \max_{\vec u \in {\cal V}^{\perp}} \frac{\vec u^T \matr M' \vec u}{\snorm{2}{\vec u}}\leq \lambda_{k+1}\leq \eps \;,$$
  because the eigenvalues of the eigenvectors of $\matr M'$ in ${\cal V}^{\perp}$ are less than $\eps$, which implies that
${\vec r^{(i)}}^T \matr M' \vec r^{(i)}$ $ \leq \eps \snorm{2}{\vec r^{(i)}}^2 \;.$
  We also have
  $
  {\vec  r^{(i)}}^T \matr A' \vec r^{(i)}  \geq  \E_{t\sim \normal}[\phi(t)(t^2-1)]\alpha_i {\vec r^{(i)}}^T \vec w^{(i)} {\vec
  w^{(i)}}^T \vec r^{(i)}
  = C_1\alpha_i \cdot \left(1-\snorm{2}{{\vec v^{(i)}}}^2
  \right)^2
  = C_1\alpha_i \snorm{2}{\vec r^{(i)}}^4 \;,$
  where the last equality follows from the Pythagorean theorem.  Therefore,
  $$\snorm{2}{\vec r^{(i)}}^2 \eps \geq {\vec r^{(i)}}^T \matr M' \vec  r^{(i)}
  \geq {\vec  r^{(i)}}^T \matr A' \vec r^{(i)} -\eps \snorm{2}{\vec r^{(i)}}^2
  \geq C_1\alpha_i \snorm{2}{\vec r^{(i)}}^4 -\eps \snorm{2}{\vec r^{(i)}}^2\;.$$
  Thus, we obtain $\alpha_i \snorm{2}{\vec r^{(i)}}^2 \leq 2\eps/C_1\;. $
  The bound now follows directly from~\eqref{eq:fourier_inequality}
  since $2 \alpha_i (1- \langle\vec w^{(i)},\vec v^{(i)}\rangle)
  = \alpha_i \snorm{2}{\vec w^{(i)} - \vec v^{(i)}}^2 =
  \alpha_i \snorm{2}{\vec r^{(i)}}^2 \leq 2 \eps/C_1$.
 \end{proof}

Now we have all the ingredients to complete our proof.
Since the dimension of the subspace that we have learned is at most $k$,
we can construct a grid with $(k/\eps)^{O(k)}$ candidates that contains an approximate solution.
Our full algorithm is summarized as Algorithm~\ref{alg:nn_learner}.
The proof of Theorem~\ref{thm:alg} follows from the above discussion and can be
found in Appendix~\ref{app:upper_bound} (Theorem~\ref{thm:appendix}.
\begin{algorithm}[H])
  \caption{Learning One-Hidden-Layer Networks with Positive Coefficients and Lipschitz Activations}
  \label{alg:nn_learner}
  \begin{algorithmic}[1]
    \Procedure{NNLearner}{$k, \eps$}
    \Comment{$k$: number of rows of weight matrix $\matr W$, $\eps$: accuracy.}
    \State Draw $m = d\, \poly( k, 1/\eps)$ samples, $(\vec x^{(i)}, y^{(i)})$, to estimate $\widehat{\matr M}$.\Comment{Lemma~\ref{lem:empirical_chow_parameters}}
\State Find the SVD of $\widehat{\matr M}$ to obtain the $k$ eigenvectors
    $\vec v^{(1)}, \ldots, \vec v^{(k)}$ that correspond to the $k$ largest
    eigenvalues, and let $\mathcal{V}$ be the subspace spanned by these vectors.
    \State Draw $m'=O(k L^2)$ samples and compute an estimation $\hat{\mu}$ of the expectation of $f(x)$
  
    \State Let $\mcal G$ be an $\eps/k$-cover of a $k$-ball wth radius $(\hat{\mu}+c\sigma)^2$ over $\cal V$,
    with respect the $\ell_2$-norm. 
    \State Draw $n = \poly(k,1/\eps)$ fresh samples $(\vec x^{(i)}, y^{(i)})$.
    \State For every $\matr U = (\vec u^{(1)}, \ldots, \vec u^{(k)}) \in \mathcal{G}^k$, let $f_{\matr U}= \sum_{i=1}^k \snorm{2}{ \vec u^{(i)}} \relu\big(\dotp{\vec u^{(i)}}{\vec x}/\snorm{2}{ \vec u^{(i)}}\big)$ and compute
    $
    e_{\matr U} = \frac{1}{n} \sum_{i=1}^n \left(f_{\matr U}(\vec x^{(i)}) - y^{(i)}\right)^2	$\label{alg:estimator}
    \State Output the candidate $f_{\matr U}$ which minimizes its corresponding error
    $e_{\matr U}$.
    \EndProcedure
  \end{algorithmic}
\end{algorithm}

 \section{Statistical Query Lower Bound} \label{sec:sq}

We start by formally defining the class of algorithms for which our lower bound
applies. In the standard statistical query model, we do not have direct access
to samples from the distribution, but instead can pick a function $q$ and get
an approximation to its expected value.  In this work, we consider algorithms
that have access to correlational statistical queries, which are more
restrictive and are defined as follows.  We remark that in the following
definition of inner product queries we do not assume that the concept $f(\bx)$
is bounded pointwise but only in the $L_2$ sense.  The properties that we shall
need for our result hold also under this weaker assumption.
\nnew{
\begin{definition}[Correlational/Inner Product Queries]
  Let $\D$ be a distribution over some domain $X$ and let $f:X\mapsto \R$,
  where $\E_{\bx \sim \D}[f^2(\bx)] \leq 1$. An inner product query is specified
  by some function $q: X \mapsto [-1,1]$ and a tolerance $\tau>0$, and returns
  a value $u$ such that $u\in [\E_{\bx \sim \D}[q(\x)f(\x)]-\tau,\E_{\bx \sim \D}[q(\x)f(\x)]+\tau]$.
\end{definition}

}
We will prove that almost any reasonable choice of activations $\sigma$, $\phi$
defines a family of functions that is hard to learn. More precisely, for a pair
of activations $\sigma, \phi$, we define the following function $f_{\sigma,
\phi}:
\R^2 \to \R$:
\begin{equation}
  \label{eq:2_dim_concept}
  f_{\sigma, \phi}(x,y) =
  \sigma\lp( \sum_{m=1}^{2k} (-1)^{m}
  \phi\left(x \cos\big(\frac{\pi m}{k}\big) +
    y \sin\big(\frac{\pi m}{k}\big)
  \right) \rp)\,.
\end{equation}
We are now ready to define the conditions on the activations $\sigma, \phi$
that are needed for our construction.  We define
\begin{align}
\mcal{H} =
&\Big\{
  f_{\sigma, \phi} \ :\ \sigma \text{ is odd }  \text{ and } ~ f_{\sigma, \phi} \not \equiv 0
\Big \} \;,  \label{eq:bad_activations}
\end{align}
where the second condition means that $ f_{\sigma, \phi}(x,y)$ \emph{as a
function of $x,y$} is not identically zero.  We can now define the class of
(normalized) functions on $\R^d$ for which our lower bound holds.
Given a set $\mcal{W}$ of $2 \times d$ matrices,
we can embed $f_{\sigma, \phi}$ into $\R^d$ by defining the following class of functions
\begin{equation}
  \label{eq:hard_functions}
  \mcal{F}_{\sigma, \phi}^{\mcal{W}}
  = \{
  \vec x \mapsto
f_{\sigma,\phi}(\vec W \vec x)/\E_{\bx \sim \normal^d}[f_{\sigma, \phi}(\vec W \vec x)] : \vec W \in \mcal{W} \} \,.
\end{equation}

\begin{remark}
  For any $f\in \mcal{H}$, we have that $f:\R^2\rightarrow \R$. We embedded $f$
  into $\R^d$ by taking $f_{\vec W}(\vec x) = f(\vec W \vec x)$ for some
  $2\times d$ matrix $\vec W$ with orthogonal rows. We prove correlational SQ
  lower bounds against learning an approximation of the embedding plane $\vec W$
  from a function $f_{\vec W}$. This will imply a lower bound against learning
  $f_{\vec W}$
  so long as the function does not vanish identically. However, this is not an
  entirely trivial condition.
  For example, if $\phi$ is a polynomial of degree less than $k$, this will happen.
  However, as we show in Appendix~\ref{sec:iterpolation}, this is essentially
  the only way that things can go wrong. In particular, so long as $\phi$ is not
  a low degree polynomial and the parity of $k$ is chosen appropriately, this
  function $f$ will not vanish, and our lower bounds will apply.
\end{remark}

\begin{theorem}[Correlational SQ Lower Bound] \label{thm:sq_theorem}
Let $\sigma, \phi$ be activations such that $f_{\sigma, \phi} \in \mcal{H}$ (see Eq.~\eqref{eq:bad_activations}).  There exists a set
$\mcal W$ of matrices $\matr W \in \R^{2 \times d}$ such that for all 
$f \in \mcal{F}_{\sigma, \phi}^{\mcal{W}}$ (see Eq.~\eqref{eq:hard_functions}) 
$\E_{\bx \sim \normal^d}[f^2(\bx)] = 1$ and the following holds: 
Any correlational SQ learning algorithm that for every concept 
$f \in {\mcal F}_{\sigma, \phi}^{\mcal{W}}$ learns a hypothesis $h$ such
that $\E_{\bx \sim \normal^d}[ (f(\bx) - h(\bx))^2 ] \leq \eps $, 
where $\eps>0$ is some sufficiently  small constant, 
requires either $2^{d^{\Omega(1)}}$ inner product queries or at least one 
query with tolerance $d^{-\Omega(k)}+2^{-d^{\Omega(1)}}$.
\end{theorem}

To prove our lower bound we will use an appropriate notion of SQ dimension.
\nnew{Specifically, we define the Correlational SQ Dimension 
that captures the difficulty of learning a class $\cal C$.
\begin{definition}[Correlational Statistical Query Dimension]
	Let $\rho >0$, let $\D$ be a probability distribution over some
	domain $X$, and let $\cal C$ be a family of functions $f:X\mapsto \R$. We denote by $\rho(\cal C)$ the average
	pairwise correlation of any two functions in $\cal C$, that is
	$
	\rho(\mcal C) = \frac{1}{|\mcal C|^2} \sum_{g, r \in \mcal C} \E_{\x\sim \D}[g(\vec x) \cdot r(\vec x)]
	$.
	The correlational statistical dimension of $\cal C$ relative to $\D$ with average
	correlation, denoted by $\text{SDA}({\cal C},\D,\rho)$, is defined to
	be the largest integer $m$ such that for every subset ${\cal C}'
	\subseteq \cal C$ of size at least $|{\cal C}'|\geq |{\cal C}|/m$, we
	have $\rho({\cal C}')\leq \rho$.
\end{definition}
}

The following lemma relates the Correlational Statistical Query Dimension of a concept
class with the number of correlational statistical queries needed to learn it.  The
difficulty lies in creating a large family of functions with small average
correlation.
  We will use the following result that translates correlational statistical dimension to
  a lower bound on the number of inner product queries needed to learn the
  function $f \in \mcal{C}$.  We note that in this paper we consider
  inner-product queries of the form $g(x) y $ where $y$ is not necessarily
  bounded.  In fact, the proof of the following lemma does not require $g(x) y$
  to be pointwise bounded (bounded $L_2$ norm is sufficient) as it can be seen
  from the arguments in \cite{szorenyi2009characterizing}, \cite{GGJKK20}, \cite{VempalaW19}.

\begin{lemma} \label{theorem:vem}
Let $\D$ be a distribution on a domain $X$ and let $\cal C$ be
a family of functions $f:X\mapsto \R$. Suppose for some $m, \tau >0$, 
we have $\textsc{SDA}({\cal C},\D,\tau)\geq m$ and assume that for all $f\in \cal C$, 
$1\geq \E_{\bx \sim \D}[f^2(\bx)] > \eta^2$. Any SQ learning algorithm that is allowed 
to make only inner product queries and for any $f\in \cal C$
outputs some hypothesis $h$ such that
$\E_{\vec x \sim \D}[(h(\vec x) - f(\vec x))^2] \leq c \, \eta^2$, where $c>0$ is a sufficiently small
constant, requires at least $\Omega(m)$ queries of tolerance $\sqrt{\tau}$.
\end{lemma}

We will require the following technical lemma, whose proof relies on Hermite polynomials, 
and can be found in Appendix~\ref{app:hermite_polynomials} (Lemma~\ref{lem:bound_coleration_app}).
\begin{lemma}\label{lem:bound_coleration} 
  Let $p(\vec x): \R^2 \mapsto \R$ be a function  and let $\matr U, \matr V \in
  \R^{2 \times d}$ be linear maps such that $\matr U \matr U^T = \matr V \matr
  V^T = \matr I \in \R^{2 \times 2}$.  Then,
  $
  \E_{\vec x \sim \normal^d}[p(\matr U \vec x) p(\matr V \vec x)] \leq
  \sum_{m=0}^{\infty} \snorm{2}{\nnew{\matr U \matr V^T}}^m \E_{\vec x \sim \normal^d}[(p^{[m]}(\vec x))^2] .
  $
\end{lemma}
In the following simple lemma, we show that two random $2$-dimensional
subspaces in high dimensions are nearly orthogonal.  In particular, we
can have an exponentially large family of almost orthogonal planes.
For the proof see Appendix~\ref{app:hermite_polynomials} (Lemma~\ref{lem:sq_rota_app}).

\begin{lemma}\label{lem:size_of_rotations} 
For any $0<c<1/2$, there exists a set $S$ of at least
  $2^{\Omega(d^c)}$ matrices in $\R^{2 \times d}$ such that for each pair
  $\matr A, \matr B \in S$, it holds $\snorm{2}{\nnew{\matr A \matr B^T}} \leq O(d^{c-1/2})$.
\end{lemma}

The following lemma shows that 
the correlation of any function $f$
of $\mcal{H}$ with any low-degree
polynomial is zero.
For the proof see Appendix~\ref{app:hermite_polynomials} (Lemma~\ref{lem:function_low_ap}).
\begin{lemma}\label{lem:function_low}
Let $f_{\sigma, \phi} \in \mcal{H}$.
For every polynomial $p(\x)$ of degree at most $k$, it holds $\E_{\x\sim \D}[f_{\sigma,\phi}(\x) \cdot p(\x)]=0$.
\end{lemma}
We are now ready to prove our main result.
\begin{proof}[Proof of Theorem~\ref{thm:sq_theorem}]
	Let $f: \R^2 \mapsto \R$ from Lemma~\ref{lem:function_low}.  Let
	$c>0$  and fix a set $\cal W$ of matrices in $\R^{2\times d}$ satisfying
	the properties of Lemma~\ref{lem:size_of_rotations}. We consider the class of functions $F_{\sigma,\phi}^{ \mcal{W}}$ (see Eq.~\eqref{eq:hard_functions}). In particular, for all $\matr
	A_i,\matr A_j \in \cal W$,
        \nnew{ let functions $G_{i}(\vec x) = f(\matr A_i \vec x)/\sqrt{\E_{\x\sim \normal^2}[f^2(\x)]} $
        and $G_{j}(\vec x)=f(\matr A_j \vec x)/\sqrt{\E_{\x\sim \normal^2}[f^2(\x)]} $.
        Notice that since $\matr A_i \matr A_i^T = \matr I$ we have that
        $\E_{\bx \sim \normal^d}[G_{i}^2(\bx)] = 1$ for all $i$.
      }
        The pairwise correlation of $G_i$ and $G_j$ is
	\begin{align}
\rho(G_i,G_j)= \frac{	\E_{\vec x \sim \normal^d}[G_{i}(\vec x) G_{j}(\vec
	x)]}{\E_{\x\sim \normal^2}[f^2(\x)] } \;,
\label{eq:pairwise_corelation}
\end{align}
where in the second equality we used that Gaussian distributions are
invariant under rotations and in the last that the expectation of
$p(\x)$ is zero.  Then, using
	Lemma~\ref{lem:bound_coleration}, it holds
	\begin{align}
	\E_{\vec x \sim \normal^d}[G_{i}(\vec x) G_{j}(\vec x)]
	&=	\E_{\vec x \sim \normal^d}[f(\matr A_i \vec x) p(\matr A_j \vec x)]
	\leq \sum_{m>k} \snorm{2}{\nnew{\matr A_i \matr A_j^T}}^m
	\E_{\vec x \sim \normal^2}[(f^{[m]}(\vec x))^2]\nonumber
	\\
	&\leq \snorm{2}{\nnew{\matr A_i \matr A_j^T}}^{k+1}
	\sum_{m>k}
	\E_{\vec x \sim \normal^2}[(f^{[m]}(\vec x))^2]
	\leq    \snorm{2}{\nnew{\matr A_i \matr A_j^T}}^{k+1}
	\E_{\vec x \sim \normal^2}[(f(\vec x))^2]\nonumber
	\\
	&\leq O(d^{k(c-1/2)})
	\E_{\vec x \sim \normal^2}[(f(\vec x))^2] \;, \label{eq:theorem_prove}
	\end{align}
	where in the first inequality we used that the first $k$ moments are
	zero, in the second the fact that the spectral norm of these  two
	matrices is less than one, and in the third inequality we used
        Parseval's theorem. Thus, using
        Equation~\eqref{eq:theorem_prove} into
	Equation~\eqref{eq:pairwise_corelation}, we get that the pairwise
	correlation is less than $\tau=O(d^{k(c-1/2)})$. Thus, from a straighforward calculation, the average
        correlation of the set $\mcal F_{\sigma, \phi}^{ \mcal{W}}$ is less $\tau + \frac{1-\tau}{|\mcal F_{\sigma,\phi}^{ \mcal{W}},|}\leq \tau +{ |\mcal F_{\sigma,\phi}^{ \mcal{W}}}|^{-1}\leq \tau +2^{-\Omega(d^{c})} $.
        Moreover, for $\tau'=d^{O(k(c-1/2))}+2^{-\Omega(d^{c})}$, the $\textsc{SDA}(\mcal F_{\sigma,\phi}^{ \mcal{W}},\D,\tau')=2^{\Omega(d^{c})}$ and the
	result follows from Lemma~\ref{theorem:vem}.
\end{proof}

\section{Conclusions and Future Directions} \label{sec:conc}
In this paper, we studied the problem of PAC learning one-hidden-layer neural networks with $k$ hidden units 
on $\R^d$ under the Gaussian distribution.
For the case of positive coefficients, we gave a polynomial time learning algorithm 
for $k$ up to $\tilde{O}(\sqrt{\log d})$.
On the negative side, we showed that no such algorithm is possible 
for unrestricted coefficients in the Correlational SQ model.

This work is part of an extensive recent literature on designing provable algorithms for learning simple families of 
neural networks. In the context of one-hidden-layer networks, a number of concrete open questions remain: 
Can we improve the dependence on $k$ in the running time to polynomial? 
Can we design learning algorithms that succeed under less stringent distributional assumptions? 
We believe that progress in both these directions is attainable.

\paragraph{Acknowledgements} 
We thank the authors of~\cite{GGJKK20} for useful comments 
that helped us improve the presentation of our lower bound proof.

\bibliography{allrefs}

\appendix

\section*{Appendix}

\section{Omitted Proofs from Section~\ref{sec:alg}}
\label{app:upper_bound}
In the following simple fact, we compute the degree-$1$ and degree-$2$ Chow parameters
of a one-layer network.
\newcommand{\activation}{\relu}

\begin{fact}[Low-degree Chow parameters]\label{fct:chow_app}
	Let $f(\vec x) = \sum_{i=1}^k \alpha_i \activation\big(\dotp{\vec w^{(i)}}{\vec x}\big)$.  Then
	 $$
	\E_{\vec x \sim \normal^d}\lp[f(\vec x)\right]
	= B_1\sum_{i=1}^k \alpha_i
	\quad \quad
	\E_{\vec x \sim \normal^d}\lp[f(\vec x) \vec x\right]
	= C \sum_{i=1}^k \alpha_i \vec w^{(i)}
	$$

	$$\E_{\x\sim \normal^d}[f(\x) (\x\x^T -\matr I)]= B \sum_{i = 1}^k \alpha_i \bw^{(i)} {\bw^{(i)}}^T $$
	where $B_1= \E_{t \sim \normal}[\relu(t)]$ ,$C= \E_{t \sim \normal}[\relu(t)t]$ and $B= \E_{t \sim \normal}[\relu(t)(t^2-1)]$.
\end{fact}
\begin{proof}
	 \begin{align*}
	\E_{\vec x\sim \normal^d}[f(\vec x)]
	&= \sum_{i=1}^k\alpha_i   \E_{\vec x\sim \normal^d}\left[ \relu\lp(\dotp{\vec w^{(i)}}{\vec x}\rp)\right]
	\\&=\sum_{i=1}^k\alpha_i\int_{\R^d}\relu \dotp{\x}{\vec w^{(i)}} {\cal N}(\vec x) \d \vec x=B_1\sum_{i=1}^k \alpha_i\;,
	\end{align*}
	where in the third equality we used the fact that normal distribution is invariant under rotations.
	 For the second equality, we have
	\begin{align*}
	\E_{\vec x\sim \normal^d}[f(\vec x) \vec x]
	&= \sum_{i=1}^k\alpha_i   \E_{\vec x\sim \normal^d}\left[ \relu\lp(\dotp{\vec w^{(i)}}{\vec x}\rp)\vec x\right]
	\\&=\sum_{i=1}^k\alpha_i\int_{\R^d} \max\lp(\dotp{\vec x}{\vec w^{(i)}},0\rp)\vec x {\cal N}(\vec x) \d \vec x
	\\
	&= \sum_{i=1}^k\alpha_i \matr R_i^{-1}\int_{\R^d} \max\lp(\dotp{\vec x}{\vec e_1},0\rp)\vec x
	{\cal N }(\vec x) \det(J(\matr R_i)) \d  \vec x \\
	&= \sum_{i=1}^k\alpha_i\vec w^{(i)}/2 \;,
	\end{align*}
	where $\matr R_i$ is some rotation matrix that maps $\vec w^{(i)}$ to $ \vec e_1$, and $J$ is the Jacobian of
	this rotation which has always determinant of 1$.$
	The Chow parameters of degree-$2$ are given by
	\begin{align*}
	\E_{\vec x\sim \normal^d}[f(\vec x)( \vec x \vec x^T-\matr I)]
	&= \sum_{i=1}^k \alpha_i   \E_{\vec x\sim \normal^d}\left[\activation\lp(\dotp{\vec w^{(i)}}{\vec x}\rp)(\vec x \vec x^T -\matr I)\right]\\
	&= \sum_{i=1}^k\alpha_i\int_{\R^d} \activation\lp(\dotp{\vec x}{\vec w^{(i)}}\rp)\vec (\x \vec x^T-\matr I) {\cal N}(\vec x) \d \vec x
	\\
	&= \sum_{i=1}^k\alpha_i \matr R_i^{-1}\int_{\R^d} \activation\lp(\dotp{\vec x}{\vec e_1}\rp)(\vec x \vec x^T-\matr I) {\cal N}(\vec x) 
	\det(J(\matr R_i)) \d  \vec x  {\matr R_i^{-1}}^T\\
	&= \sum_{i=1}^{k} \alpha_i \matr R_i^{-1}\int_{\R^d}\sum_{k,l=1}^d \activation\lp(\x_1\rp)(\x_k\x_l-\delta_{k, l}) \vec e_k \vec e_l^T {\cal N}(\vec x)  \d  \vec x  {\matr R_i^{-1}}^T
		\end{align*}
There are four cases. The first case is when $k\neq l\neq 1$. By independence, we have that 
$\E_{\vec x\sim \normal^d}[\phi(\x_1)(\x_k\x_l-\delta_{k, l})] = \E_{\vec x\sim \normal^d}[\phi(\x_1)] \E_{\vec x\sim \normal^d}[(\x_k\x_l)]=0$, 
where we used the independence of the random variables $\x_k, \x_l$. 
Similarly, if $k\neq l$ and $k=1$ we have that $\E_{\vec x\sim \normal^d}[\phi(\x_1)(\x_1\x_l)]=0$. 
If $k=l\neq 1$, then $\E_{\vec x\sim \normal^d}[\phi(\x_1)(\x_l^2-1)=0$, because $\E_{\vec x\sim \normal^d}[\x_l^2]=1$. 
Thus, the only non-zero case is when $k=l=1$. Then, we obtain
\begin{align*}
\E_{\vec x\sim \normal^d}[f(\vec x)( \vec x \vec x^T-\matr I)]
&= \sum_{i=1}^{k} \alpha_i \matr R_i^{-1}\int_{\R^d} \activation\lp(\x_1\rp)(\x_1^2-1) \vec e_1 \vec e_1^T {\cal N}(\vec x)  \d  \vec x  {\matr R_i^{-1}}^T\\
&= B \sum_{i = 1}^k \alpha_i \bw^{(i)} {\bw^{(i)}}^T \;.
\end{align*}
\end{proof}

\begin{lemma}\label{lem:bound_error}
	Let $f_{\vec \alpha,\matr W}(\vec x) = \sum_{i=1}^k \alpha_i
	\relu(\dotp{\vec w^{(i)}}{\x})$ and $f_{\vec \beta,\matr V}(\vec x) = \sum_{i=1}^k \beta_i
	\relu(\dotp{\vec v^{(i)}}{\x})$  with $\alpha_i,\beta_i$ $>0$, then it holds
$
	\E_{\vec x\sim \normal^d}[(f_{\vec \alpha,\matr W}(\vec x) -
	f_{\vec \beta,\matr V}(\vec x))^2]\leq 2 k
	\E_{t\sim \normal}[(\phi'(t))^2] \sum_{i=1}^k \alpha_i^2\snorm{2}{\vec v^{(i)} -\vec w^{(i)}  }
	+  k	\E_{t\sim \normal}[\phi(t)^2] \sum_{i=1}^k (\alpha_i-\beta_i)^2 \;.
$
\end{lemma}
\begin{proof}
	We have
	\begin{align*}
	\E_{\vec x\sim \normal^d}[(f_{\vec \alpha,\matr W}(\vec x) -
	f_{\vec \beta,\matr V}(\vec x))^2]
	\leq k 	\E_{\vec x\sim \normal^d}\left[\sum_{i=1}^k\left(\alpha_i \relu\Big(\dotp{\vec w^{(i)}}{\x}\Big) -\beta_i \relu\Big(\dotp{\vec v^{(i)}}{\x}\Big)\right)^2\right]\\
	\leq  k 	\E_{\vec x\sim \normal^d}\left[\sum_{i=1}^k\alpha_i\left( \relu\Big(\dotp{\vec w^{(i)}}{\x}\Big) - \relu\Big(\dotp{\vec v^{(i)}}{\x}\Big)\right)^2\right] + 	\E_{\vec x\sim \normal^d}\left[\sum_{i=1}^k \relu\Big(\dotp{\vec v^{(i)}}{\x}\Big)^2(\alpha_i  -\beta_i )^2\right]\\
	\leq 2 k
	\E_{t\sim \normal}[(\phi'(t))^2] \sum_{i=1}^k \alpha_i^2\snorm{2}{\vec v^{(i)} -\vec w^{(i)}  }
	+  k	\E_{t\sim \normal}[\phi(t)^2]\sum_{i=1}^k (\alpha_i-\beta_i)^2 \;,
	\end{align*}
	where we used Lemma~\ref{lem:dimension_reduction}.
\end{proof}

\begin{fact}\label{lem:bound_expect} Let  $f(\vec x)=\sum_{i=1}^k \alpha_i \relu(\dotp{\vec w_i}{\x})$ and $y=f(\x) + \xi$ where $\xi$ is zero mean subgaussian with variance $\sigma^2$. Let $B_2=\E_{t\sim \normal}[\relu^2(t)]$ and $c>0$ a constant, then using $O(k B_2 )$ samples we can find $\hat{\mu}$  such as
	$$ \E_{\x\sim \normal^d}[f(\x)]\leq 2\hat{\mu} + 2c\sqrt{\frac{\sigma^2}{k}} \quad \text{and}\quad  \hat{\mu}\leq  \frac{3}{2}\E_{\x\sim \normal^d}[f(\x)] + c\sqrt{\frac{\sigma^2}{k}}$$ with probability at least 3/4.
\end{fact}
\begin{proof}
	Let $\hat{\mu} = \frac{1}{m}\sum_{i=1}^m y^{(i)} $, then from Chebyshev's inequality, we have
	\begin{align*}
	\pr[|\hat{\mu}-\E_{\x\sim \normal^d}[f(\x)]|\geq 2\sqrt{ \var[y]/m}] \leq 1/4 \;.
	\end{align*}
	Thus with probability $3/4$, it holds
\begin{align*}
	|\hat{\mu}-\E_{\x\sim \normal^d}[f(\x)]|\leq 2\sqrt{ \var[y]/m}&\leq 2 \sqrt{\frac{\E_{\x \sim \normal^d}[f^2(\x)]}{m} }+2\sqrt{\frac{\sigma^2}{m}}\\
	&\leq 2 \E_{\x \sim \normal^d}[f(\x)]\sqrt{\frac{k B_2}{m} }+2\sqrt{\frac{\sigma^2}{m}}\;,
\end{align*}
	 to get last inequality we used Cauchy–Schwarz. Taking $m=O(k B_2)$ we get 	$\frac{1}{2} \E_{\x\sim \normal^d}[f(\x)]\leq \hat{\mu} + c\sqrt{\frac{\sigma^2}{k}} $ and $\hat{\mu}\leq  \frac{3}{2}\E_{\x\sim \normal^d}[f(\x)] + c\sqrt{\frac{\sigma^2}{k}} $.
\end{proof}

\begin{lemma}\label{lem:bound_square}
Let  $f(\vec x)=\sum_{i=1}^k \alpha_i \relu(\dotp{\vec w_i}{\x})$, 
$B_2=\E_{t\sim \normal}[\relu^2(t)]$ and  $B_4=\E_{t\sim \normal}[\relu^4(t)]$. Then
	$ \E_{\vec x\sim \normal^d}[f^4(\vec x)]\leq \frac{B_4}{B_2^2}k^2\E_{\vec x\sim
		\normal^d}[f(\vec x)^2]^2$.
\end{lemma}
\begin{proof}
	To bound $ \E_{\vec x\sim \normal^d}[f^4(x)]$, using Cauchy-Schwartz, it holds that
	\begin{align}
	\sqrt{ \E_{\vec x\sim \normal^d}[f^4(\vec x)]}
	&\leq \sum_{i=1}^k \alpha_i^2 \sqrt{\E_{\vec x\sim \normal^d}
		\Big[\Big(\sum_{i=1}^k \relu^2(\dotp{\vec w^{(i)}}{\vec x})\Big)^2\Big] }
	\leq \sum_{i=1}^k \alpha_i^2 \sqrt{k\E\Big[\sum_{i=1}^k \relu^4(\dotp{\vec w^{(i)}}{\vec x})\Big] }\nonumber
	\\
	&\leq k B_4^{1/2} \sum_{i=1}^k \alpha_i^2
	\leq k \frac{B_4^{1/2}}{B_2}\E_{\vec x\sim \normal^d}[f(\vec x)^2]\nonumber\;,
	\end{align}
	where in the last inequality we used that $\sum_{i=1}^k \alpha_i^2 B_2 \leq \E_{\vec x\sim \normal^d}[f^2(\vec x)]$.
\end{proof}

\begin{lemma}\label{lem:bound_square_error}
	Let  $f(\vec x)=\sum_{i=1}^k \alpha_i \relu(\dotp{\vec w_i}{\x})$ and $y=f(\x) + \xi$ where $\xi$ is zero mean subgaussian with variance $\sigma^2$. Moreover, let $B_2=\E_{t\sim \normal}[\relu^2(t)]$ and $B_4=\E_{t\sim \normal}[\relu^4(t)]$. Then, if $Y_u = \frac{1}{m} \sum_{i=1}^m (f_u(\x^{(i)})- y^{(i)}))^2$,we can find $\hat{Y}_u$ such that $$|\hat{Y}_u-\E_{\x\sim \normal^d}[Y_u]|\leq c\eps^2k^2\frac{B_4^{1/2}}{B_2}   \left( \E_{\vec x\sim
		\normal^d}[f^2(\vec x)]+\E_{\vec x\sim
		\normal^d}[f_u^2(\vec x)]+\sigma^2  \right)
	$$ with probability $1-\delta$ with $O(\frac{1}{\eps^4}\log(1/\delta))$ samples, where $c$ is a universal constant.
\end{lemma}
\begin{proof}
	Let $Y= (f_u(\x)- y)^2= (f_u(\x)- f(\x))^2 + y^2 -2y (f_u(\x)- f(\x)) $. Then the variance of each term is
	\begin{align*}
	\var[(f_u(\x)- f(\x))^2]&\leq \E_{\vec x \sim \normal^d}[(f_u(\x)- f(\x))^4]\leq 4\E_{\vec x \sim \normal^d}[(f_u^2(\x)+ f^2(\x))^2]\\&\leq 8\E_{\vec x \sim \normal^d}[f_u^4(\x)+ f^4(\x)]
	\leq  8\frac{B_4}{B_2^2}k^2\big(\E_{\vec x\sim
		\normal^d}[f^2(\vec x)]^2+ \E_{\vec x\sim
		\normal^d}[f_u^2(\vec x)]^2 \big)\;,
	\end{align*}
	where in the third and in the fourth inequality we used that $(a\pm b)^2\leq 2a^2 + 2b^2$  and in the last one we used Lemma ~\ref{lem:bound_square}.
	Thus,
	\begin{align*}
	\var[Y]&\leq 8\frac{B_4}{B_2^2}k^2\big(\E_{\vec x\sim
		\normal^d}[f^2(\vec x)]^2 +\E_{\vec x\sim
		\normal^d}[f_u^2(\vec x)]^2 \big) + 16e^2\sigma^4 + 2\sigma^2 \big(\E_{\vec x\sim
		\normal^d}[f^2(\vec x)] +\E_{\vec x\sim
		\normal^d}[f_u^2(\vec x)] \big) \\
	&\leq 
	ck^4 \frac{B_4}{B_2^2}\left( \E_{\vec x\sim
		\normal^d}[f^2(\vec x)]+\E_{\vec x\sim
		\normal^d}[f_u^2(\vec x)]+\sigma^2  \right)^2
	\;,
	\end{align*}
	where $c$ is a universal constant.
	From Chebyshev's inequality, we have that we need $m=O(\frac{1}{\eps^4})$ for an error at most $\sqrt{c}\eps^2k^2(\frac{B_4^{1/2}}{B_2}   \left( \E_{\vec x\sim
		\normal^d}[f^2(\vec x)]+\E_{\vec x\sim
		\normal^d}[f_u^2(\vec x)]+\sigma^2  \right)$. Then using the median trick, we can boost the confidence to $1-\delta$ with $m\log(1/\delta)$ samples.
\end{proof}

Since the dimension
of the subspace, that we have learned, is at most $k$, the following standard
lemma gives us that a grid with $(k/\eps)^{O(k)}$ candidates suffices.

\begin{lemma}[Corollary 4.2.13 of \cite{vershynin2018high}]\label{lem:size_of_cover}
  There exists  be an $\eps$-cover of the unit ball in $\R^k$, with respect
  the $\ell_2$ norm, of size at most $(1+2/\eps)^k$.
\end{lemma}

We are now ready to prove our main theorem, Theorem~\ref{thm:alg}, which
we restate for convenience.

\begin{theorem}[Learning Sums of Lipschitz Activations]\label{thm:appendix}
   Let $f(\vec x) = \sum_{i=1}^k \alpha_i \relu\big(\dotp{\vec w^{(i)}}{\vec x}\big)$ with $\alpha_i>0$ for all $i\in[k]$,
  where $\phi(t)$ is an $L$-Lipschitz, non-negative activation function
  such that $\E_{t \sim \normal}[\phi(t)] \geq C$, $\E_{t\sim \normal}[\relu(t)(t^2-1)]\geq C$, where $C>0$ and 
  $\E_{t \sim \normal}[\phi^2(t)]$ is finite. There
  exists an algorithm that given $k\in \mathbb{N}, \eps>0$, and sample access
  to a noisy set of samples from $f: \R^d \rightarrow \R_+$, draws $m =
  d \cdot \poly(k, 1/\eps)\cdot \poly(L/C)$ samples, runs in time $\poly(m) + \wt{O}((1/\eps)^{k^2})$,
  and outputs a proper hypothesis $h$ that, with probability at least $9/10$,
  satisfies \[ \E_{\vec x \sim\normal^d}[(f(\vec x) - h(\vec x))^2]
  \leq \eps^2 \poly(L/C)
\left(\sigma^2 + \E_{\vec x \sim \normal^d}[f(\x)^2]\right)\;.
\]
\end{theorem}
\begin{proof}
  Denote $M_f = \E_{\vec x \sim \normal^d}[f(\vec x)]$ and $M_{f^2} = \E_{\vec x \sim \normal^d}[f(\vec x)^2]$.
Using Lemma~\ref{lem:empirical_chow_parameters}, we get that with $m =\wt{O}(d\,k^3 /\eps^2)$ 
  samples with high constant probability it holds that
  $\snorm{2}{\widehat{\matr M}- \E_{\vec x \sim \normal^d}[
  f(\vec x) \vec x \vec x^T]} \leq \frac{\eps}{k}( M_f\frac{L}{C} +\sigma )$. 
  From Lemma~\ref{lem:dimension_reduction}, we obtain
  that there exists a matrix $\matr V \in \R^{k \times d}$ whose rows are
  vectors of the subspace $\mathcal{V}$ such that $\E_{\vec x \sim \normal^d}[(f_{\matr
  V}(\vec x) - f(\vec x))^2] \leq 2\eps^2  \frac{L^2}{C}(M_f^2\frac{L}{C} +M_f\sigma)$. 
  From Fact~\ref{lem:bound_expect}, let $\hat{\mu}$ be an upper bound to $M_f$ 
  (that is, $\hat{\mu} \leq 2\mu + 2c\sqrt{\sigma^2/k}$, where $\mu$ is the estimated value), 
  then using Fact~\ref{fact:cover}, with the value $\hat{\mu}/k$, 
  we get an approximation of each $a_i$ with error $\eps \hat{\mu}/k$.

Using Lemma~\ref{lem:size_of_cover}, the size of a cover is 
  $|\mathcal G| \leq \left( (1+4k/\eps)^{k} \log(k\eps)/\eps\right)^k$, 
  because we need vectors with norm from $\eps M_f$ to $M_f$, our cover is created using the upper bound on $M_f$.
  We have that there exists $\matr U$ whose rows are vectors 
  in the cover $\mathcal{G}$ such that
  \begin{align}
  	\E_{\vec x \sim \normal^d}[(f_{\matr U}(\vec x) - f_{\matr V}(\vec x))^2]& \leq c( \eps^2 M_f^2 L^2 + L^2\eps^2 \hat{\mu}^2)\nonumber\\&\leq  c \eps^2  L^2(M_f^2 + M_f\sigma+ \sigma^2) \nonumber
  	\\&\leq  c \eps^2  L^2(M_f +\sigma)^2 \;,
  \end{align}
  where in the first inequality we used Lemma~\ref{lem:bound_error} and in the second one Fact~\ref{lem:bound_expect}. 
  The error of the best hypothesis(i.e., the one that minimizes the error) in the cover, will be
  \begin{align}
    \E_{\vec x \sim \normal^d}[(f_{\matr U}(\vec x) - f(\vec x))^2] &\leq
    2 \E_{\vec x \sim \normal^d}[(f_{\matr U}(\vec x) - f_{\matr V}(\vec x))^2]
    +
    2 \E_{\vec x \sim \normal^d}[(f_{\matr V}(\vec x) - f(\vec x))^2]\nonumber\\ &\leq  2c \eps^2  L^2(M_f +  \sigma)^2 + 4\eps^2 \frac{L^2}{C} (M_f^2\frac{L}{C} +M_f\sigma)\nonumber\\
    &\leq \eps^2  \poly(L/C)(M_f +  \sigma )^2\;. \label{eq:th314}
  \end{align}
  Finally, using the estimator from Line \ref{alg:estimator}, Lemma~\ref{lem:bound_square_error},
  we conclude that $m''=O(\frac{k^4}{\eps^4}\log(|\mathcal{G}|))$ samples are
  sufficient to test all the vectors of the cover $\mathcal{G}$ 
  and find the one that minimizes the error with high probability. For each element $i\in \mcal{G}$, let $e_{i}= 	\E_{\vec x \sim\normal^d}[(f(\vec x) - f_{i}(\vec x))^2] + \sigma^2$, which is the square error of the $i$-th hypothesis in $\mcal{G}$ and let $\hat{e}_i$ be the estimated value. We have with high probability that
  \begin{align}
    |\hat{e}_{i}-\E_{\vec x \sim \normal^d} [\hat{e}_{i}] | \leq \eps^2 \poly(L/C)(\sigma^2 + M_{f^2})\;, \label{eq:equaltion6}
  \end{align}
where we used  $  \E_{\vec x \sim \normal^d}[f_{\matr U}(\vec x)]\leq k \hat{\mu}\leq k M_f + 2c'\sigma \sqrt{k} $. Set $h(\x)=\argmin_{i \in \mcal{G}}   |\hat{e}_{i}-\E_{\vec x \sim \normal^d} [\hat{e}_{i}] | $, using Equations~\eqref{eq:th314} and \eqref{eq:equaltion6}, then
   \begin{align*}
   	\E_{\vec x \sim\normal^d}[(f(\vec x) - h(\vec x))^2]
   	&\leq  \eps^2 \poly(L/C)(\sigma^2 + M_{f^2}) +   \E_{\vec x \sim \normal^d}[(f_{\matr U}(\vec x) - f(\vec x))^2]
   	\\
   &\leq  \eps^2 \poly(L/C)(\sigma^2 + M_{f^2})+\eps^2 \poly(L/C) \left(\sigma +
 M_f \right)^2\\&\leq \eps^2 \poly(L/C) \left(\sigma^2 +
 M_{f^2} \right)\;,
   \end{align*}
where the last inequality follows from Jensen's inequality.
\end{proof}
\begin{fact}\label{fact:cover}
Let $\mcal G$ be a set of unit vectors of size $m$. 
We can construct a new set $\mcal G'$ of size $m\log(1/\eps)/\eps$ with the property: 
For every $\alpha\in [0,B]$ and every vector $\vec v\in \mcal G$, there exists a $\vec w\in \mcal G'$ such that 
$\snorm{2}{\alpha \vec v-\vec w}^2\leq \eps^2  B^2 $ and  $\snorm{2}{ \vec v-\frac{\vec w}{\snorm{2}{\vec w}}}^2=0 $. 
\end{fact}
\begin{proof}
	For each vector $\vec v\in \cal G$, add the vectors $(1-\eps)^i B\vec v$ for $i=0,\ldots , \log(1/\eps)/\eps$ to  $\mcal G'$. Then, for all $\alpha\in [0,B]$ and for every vector $\vec v\in \mcal G$ there exists a $\vec w \in G'$ such that $\snorm{2}{\alpha \vec v-\vec w}^2\leq\snorm{2}{(1-\eps)^{t+1} \vec v-\vec v (1-\eps)^tB}^2\leq  \eps^2  B^2 $, for a value $t$ such that $\alpha \in [(1-\eps)^{t+1}B, (1-\eps)^t B]$. 
\end{proof} \section{Empirical Estimates of Chow Parameters}
\label{app:chow_estimation}

In this section, we show that roughly $O(d k/\eps^2)$ samples are
sufficient to estimate the degree-$2$ Chow parameters in spectral norm.
We will prove the following lemma (Lemma~\ref{lem:empirical_chow_parameters} in the main body).

\begin{lemma}[Estimation of Degree-$2$ Chow parameters]\label{lem:appendix_chow_est}
 Let $ f_{\alpha, \matr W}(\vec x) = \sum_{i=1}^k
  \alpha_i \phi(\langle \vec w^{(i)}, \vec x \rangle)$,
  where $\phi(t)$ is an $L$-Lipschitz, positive activation function
  such that $\E_{t \sim \normal}[\phi(t)] \geq C$. 
  Let $\matr \Sigma = \E_{\vec x \sim \normal^d}[ f_{\alpha, \matr W}(\vec x) \vec x \otimes \vec x ]$
  be the degree-$2$ Chow parameters of $f_{\alpha, \matr W}$.
  Then for $N = \wt{O}(d k/\eps^2)$ samples $(\vec x^{(i)},
  y^{(i)})$, where $y^{(i)}=f_{\alpha, \matr W}(\vec x^{(i)}) +\xi_i$ and
  $\xi_i$ is a zero-mean, subgaussian noise with variance $\sigma^2$,
  it holds with probability at least $99\%$ that
  $$\snorm{2}{ \frac{1}{N} \sum_{i=1}^N \vec x^{(i)} \otimes \vec x^{(i)} y^{(i)} - \matr \Sigma }
  \leq \eps \left(
    \sigma + \frac{L}{C} \E_{\vec x \sim \normal^d}[f_{\alpha, \matr W}(\vec x)]
    \right) \;.
  $$
\end{lemma}

We will require the following lemma from~\cite{vershynin2010}
about concentration of matrices with heavy-tailed independent rows.
\begin{lemma}[Theorem 5.48 of \cite{vershynin2010}] \label{lem:heavy_rows_covariance}
  Let $\matr A$ be an $N \times d$ matrix whose rows $\matr A_i$ are independent random
  vectors in $\R^d$ with the common second moment matrix $\matr \Sigma = \E[\matr A_i \matr A_i^T]$.
  Let $ m  = \E[\max_{i \leq N} \snorm{2}{\matr A_i}^2] $. Then
  $$
  \E\left[\snorm{2}{\frac{1}{N} \matr A^T \matr A  - \matr \Sigma} \right]
  \leq \max(\snorm{2}{\matr \Sigma}^{1/2} \delta, \delta^2)\, ,
  \qquad \text{ where }
  \delta = C \sqrt{\frac{m \log( \min(N, d))}{N}} \;.
  $$
\end{lemma}

We are also going to use the following concentration result on sums of random
matrices.
\begin{lemma}[Rudelson's Inequality, Corollary 5.28 in \cite{vershynin2010}]
  \label{lem:rudelson}
  Let $\vec x^{(1)} ,\ldots, \vec x^{(N)}$ be fixed vectors in $\R^d$.  Let
  $\xi^{(1)}, \ldots, \xi^{(N)}$ be zero mean sub-Gaussian with variance
  $\sigma^2$ random variables.  Then
  $$
  \E\left[\Big\|\sum_{i=1}^N \xi^{(i)} \vec x^{(i)} \otimes \vec x^{(i)}\Big\|_2 \right]
  \leq
  C \sigma \sqrt{\log d} \cdot \max_{i \leq N} \snorm{2}{\vec x^{(i)}}
\Big\| \sum_{i=1}^N \vec x^{(i)} \otimes \vec x^{(i)} \Big\|_2^{1/2}\, .
  $$
\end{lemma}

We will also need the following well-known result on concentration of
polynomials of independent Gaussian random variables.  See, e.g., \cite{Don14}.

\begin{lemma}[Gaussian Hypercontractivity] \label{lem:hypercontractivity}
  Let $p(\vec x): \R^d \mapsto \R$ be a degree-$m$ polynomial.  Then
  $$
  \Prob_{\vec x \sim \normal^d} \left[\left|p(\vec x) - \E_{\vec y \sim \normal^d}[p(\vec y)] \right| > t\right]
  \leq e^2 \exp\bigg(- \Big(\frac{t^2}{C\ \var_{\vec x \sim \normal^d}[ p(\vec x)] } \Big)^{1/m} \bigg) \;,
  $$
  where $C > 0$ is an absolute constant.
\end{lemma}

\begin{proof}[Proof of Lemma~\ref{lem:appendix_chow_est}:]
  We have
  \begin{align*}
    \snorm{2}{ \frac{1}{N} \sum_{i=1}^N
    \vec x^{(i)} \otimes \vec x^{(i)} f_{\alpha, \matr W}(\vec x^{(i)}) + \xi_i) - \matr \Sigma }
  &\leq
  \snorm{2} {\frac{1}{N} \sum_{i=1}^N \vec x^{(i)} \otimes \vec x^{(i)} f_{\alpha, \matr W}(\vec x^{(i)}) - \matr\Sigma}
  \\
  &+ \snorm{2} {\frac{1}{N} \sum_{i=1}^N \xi^{(i)} \vec x^{(i)} \otimes \vec x^{(i)}}\, .
  \end{align*}

  We next bound the probability that
  $\snorm{2}{\vec x}^2 f_{\alpha, \matr W}(\vec x)$ is large.  We have
  \begin{align}
    \label{eq:norm_relu_union_bound}
    \Prob_{\vec x \sim \normal^d}[ &\snorm{2}{\vec x}^2 f_{\alpha, \matr W}
    \geq t
    ]
    =
    \Prob_{\vec x \sim \normal^d}\left[\sum_{j=1}^k
      \alpha_j \snorm{2}{\vec x}^2
      \phi\left(\dotp{\vec x}{\vec w^{(j)}} \right)
      \geq t
    \right]
    \nonumber
    \\
				   &\leq
				   \sum_{j=1}^k
				   \Prob_{\vec x \sim \normal^d}
				   \left[
				     \snorm{2}{\vec x}^2 \phi\left(\dotp{\vec x}{\vec w^{(j)}} \right)
				     \geq \frac {t}{k \sum_{j=1}^k \alpha_j}
				   \right]
				   \leq
				   k
				   \Prob_{\vec x \sim \normal^d}
				   \left[
				     \snorm{2}{\vec x}^2 |\x_1|
				     \geq \frac {t}{L k \sum_{j=1}^k \alpha_j}
				   \right]\, ,
  \end{align}
  where for the second inequality we used the union bound and for the last
  one we used the rotation invariance of the normal distribution and the Euclidean
  norm to set $\vec w^{(j)} = \vec e_1$.  Moreover, we used the fact that
  $\phi(\vec x_1) \leq L |\vec x_1|$ since $\phi(\cdot)$ is $L$-Lipschitz.
  \begin{align}
    \label{eq:norm_relu_concentration}
    \Prob_{\vec x \sim \normal^d}
    \left[ \snorm{2}{\vec x}^2 |\x_1| \geq t \right]  =
    \Prob_{\vec x \sim \normal^d}
    \left[
      \left| \snorm{2}{\vec x}^2 \x_1
      - \E_{\vec x \sim \normal^d}[\snorm{2}{\vec x}^2 \x_1]
    \right| \geq t \right]
    \leq
    \exp(2- (t^2/(C' d^2))^{1/3}) \;,
  \end{align}
  where we used Lemma~\ref{lem:hypercontractivity} and the fact
  that $
  \var_{\vec x \sim \normal^d}[ \snorm{2}{\vec x}^2 x_1]
  = \E_{\vec x \sim \normal^d}[\snorm{2}{\vec x}^4 x_1^2]
  = d^2 + 4 d + 10 \leq 15 d^2$, for all $d \geq 1$.
  Note that $C' = 15 C$, where $C$ is the absolute constant of
  Lemma~\ref{lem:hypercontractivity}.  Combining
  Equation~\eqref{eq:norm_relu_union_bound},
  Equation~\eqref{eq:norm_relu_concentration}
  and the fact that 
  $\sum_{j=1}^k \alpha_j = \E_{\vec x \sim \normal^d}[f_{\alpha, \matr W}(\vec x)]/\E_{t \sim \normal}[\phi(t)]:= B$,
  we obtain
  $$
  \Prob_{\vec x \sim \normal^d}
  [ \snorm{2}{\vec x}^2 f_{\alpha, \matr W}(\vec x) \geq t ]
  \leq k \exp(2- (t^2/(4 C' L^2 B^2 k^2  d^2))^{1/3}) \;.
  $$
  Define the random variables $y_i =
  \snorm{2}{\vec x^{(i)}}^2 f_{\alpha, \matr w}(\vec x^{(i)})
  $.
  Set $S = O(k d B L)$ and $Q = O(S \log k \log^3 N)$ we have
  \begin{align*}
    \E\left[ \max_{i \leq N} y_i
    \right]
    - Q
&=
\int_{0}^{\infty}
\Prob\left[
  \max_{i \leq N} y_i \geq t + Q
\right]
\d t
\leq
N k \int_{0}^\infty
\Prob\left[ y_1 \geq t + Q \right]
\d t
\\
&\leq
N k \int_{0}^\infty
\exp(-(t/S)^{2/3})
\d t
= \wt{O}(d k B L) \;.
  \end{align*}
Now that we have a bound on the expected maximum deviation, we can apply
  Lemma~\ref{lem:heavy_rows_covariance} with 
  $\matr A = \frac{1}{N} \sum_{i=1}^N \vec x^{(i)} \otimes \vec x^{(i)} f_{\alpha, \matr W}(\vec x^{(i)})\in \R^{N \times d}$
  and $m = \wt{O}(d k B L)$.  Since $\snorm{2}{\matr \Sigma} \leq
  (1+1/\sqrt{2 \pi}) B$, we obtain that for $N = \wt{O}(d k/\eps^2)$ it
  holds $ \E \snorm{2}{(1/N) \matr A^T \matr A - \matr \Sigma} \leq B L \eps$.

  To finish the proof, it remains to bound the norm of the sum $\frac{1}{N}
  \sum_{i=1}^N \xi^{(i)} \vec x^{(i)} \otimes \vec x^{(i)}$.
  From Lemma~\ref{lem:rudelson} we obtain that it is bounded
  above by
  \begin{align*}
    C \sigma \sqrt{\log d}
    \E_{\vec x^{(1)},\ldots, \vec x^{(N)}}
    \left[
      \max_{i \leq N} \snorm{2}{\vec x^{(i)}}
      \Big\| \sum_{i=1}^N \vec x^{(i)} \otimes \vec x^{(i)} \Big\|_2^{1/2}
    \right]\ .
  \end{align*}
  We now use Cauchy-Schwarz for the above expectation and observe that
  \begin{align*}
    \E_{\vec x^{(1)},\ldots, \vec x^{(N)}}
    \left[
      \max_{i \leq N} \snorm{2}{\vec x^{(i)}}^2
    \right]
    \leq \wt{O}(d \log N)\;,
  \end{align*}
  which follows from Lemma~\ref{lem:hypercontractivity} similarly as our
  previous bound.  Moreover, from Lemma~\ref{lem:heavy_rows_covariance} we
  obtain that
  $$
  \E_{\vec x^{(1)},\ldots, \vec x^{(N)}}
  \left[
    \Big\| \frac{1}{N}\sum_{i=1}^N \vec x^{(i)} \otimes \vec x^{(i)} \Big\|_2
  \right]
  \leq \wt{O}(\sqrt{d/N})\;.
  $$
  Putting everything together, we obtain that with $N = \wt{O}(d/\eps^2)$
  samples, the expected norm of $\frac{1}{N} \sum_{i=1}^N \xi^{(i)} \vec
  x^{(i)} \otimes \vec x^{(i)}$ is at most $O(\sigma \eps)$.  The result now
  follows from Markov's inequality.
\end{proof}

\section{Omitted Proofs from SQ Lower Bound} \label{app:lower_bound}

\subsection{Preliminaries: Multilinear Algebra}
\label{app:multilinear_algebra}
Here we introduce some multilinear algebra notation.
An order $k$ tensor
$\matr A$ is an element of the $k$-fold tensor product of subspaces $\matr A
\in \mathcal{V}_1 \otimes \ldots \otimes \mathcal{V}_k$.
We will be exclusively working with subspaces of $\R^d$ so a tensor $A$ can
be represented by a sequence of
coordinates, that is $A_{i_1,\ldots,i_k}$.   The
tensor product of a order $k$ tensor $\matr A$ and an order $m$ tensor $\matr
B$ is an order $k + m$ tensor defined as $(\matr A \otimes \matr
B)_{i_1,\ldots, i_k,j_1,\ldots,j_m} = \matr A_{i_1,\ldots,i_k} \matr
B_{j_1,\ldots, j_m}$.  We are also going to use capital letters for
multi-indices, that is tuples of indices $I = (i_1,\ldots, i_k)$.
We denote by $E_i$ the multi-index that has $1$ on its $i$-th co-ordinate and
$0$ elsewhere.
For example the previous tensor product can be denoted as $\matr A_I \matr B_J$
To simplify notation we are also going to use Einstein's summation where we
assume that we sum over repeated indices in a product of tensors.  For example
if $\matr A \in \R^d \otimes \R^d$, $\vec v \in \R^d$, $\vec u \in \R^d$ we
have $\sum_{i,j=1}^d \matr \bv_i \vec{u}_j \matr A_{ij} = \bv_i \vec{u}_j \matr
A_{ij}$.  We define the dot product of two tensors (of the same order) to be
$\langle \matr A, \matr B \rangle = \matr A_{i_1,\ldots,i_k} \matr
B_{i_1,\ldots, i_k} = \matr A_I \matr B_I$.  We also denote the $\ell_2$-norm
of a tensor by $\snorm{2}{\matr A} = \sqrt{\dotp{\matr A}{\matr A}}$.  We
denote by $\matr A(\vec X)$ a function that maps the tensor $\vec X$ to a
tensor $\matr A (\vec X)$.
Let $\mathcal{V}$ be a vector space and let $\matr A(\vec x): \R^d \mapsto
{\mathcal{V}}^{\otimes k}$ be a tensor valued function.  We denote by
$\partial_i \matr A(\vec x)$ the tensor of partial derivatives of $A(\vec x)$,
$\partial_i \matr A(\vec x) = \partial_i \matr A_J(\vec x) $ is a tensor of
order $k+1$ in $\mathcal{V}^{\otimes k} \otimes \R^d$.  We also denote this tensor
$
\nabla \matr A(\vec x) = \partial_i \matr A_J(\vec x).
$
Similarly we define higher-order derivatives, and we denote
$$
\nabla^m \matr A(\vec x) = \partial_{i_1}  \ldots \partial_{i_m}
\matr A_J(\vec x)
\in \mathcal{V}^{\otimes k} \otimes (\R^d)^{\otimes m}
$$

\subsection{Preliminaries: Hermite Polynomials}
\label{app:hermite_polynomials}
We are also going to use the Hermite polynomials that form a orthonormal system
with respect to the Gaussian measure.  We denote by $L^2(\R^d, \normal)$ the
vector space of all functions $f:\R^d \to \R$ such that $\E_{\vec x \sim
\normal^d}[f^2(\x)] < \infty$.  The usual inner product for this space is
$\E_{\vec x \sim \normal^d}[f(\vec x) g(\vec x)]$.
The $L_2$ norm of a function $f$ is then defined as $\snorm{2}{f} = \sqrt{
\E_{\vec x \sim \normal^d}[f^2(\x)]}$.
While, usually one considers the probabilists's or physicists' Hermite polynomials,
in this work we define the \emph{normalized} Hermite polynomial of degree $i$ to be
\(
H_0(x) = 1, H_1(x) = x, H_2(x) = \frac{x^2 - 1}{\sqrt{2}},\ldots,
H_i(x) = \frac{He_i(x)}{\sqrt{i!}}, \ldots
\)
where by $He_i(x)$ we denote the probabilists' Hermite polynomial of degree
$i$.  These normalized Hermite polynomials form a complete orthonormal basis
for the single dimensional version of the inner product space defined above. To
get an orthonormal basis for $L^2(\R^d, \normal^d)$, we use a multi-index $J\in
\N^d$ to define the $d$-variate normalized Hermite polynomial as $H_J(\vec x) =
\prod_{i=1}^d H_{v_i}(\x_i)$.  The total degree of $H_J$ is $|J| = \sum_{v_i \in
J} v_i$.  Given a function $f \in L^2$ we compute its Hermite coefficients as
\(
\hat{f}(J) = \E_{\vec x\sim \normal^d} [f(\vec x) H_J(\vec x)]
\)
and express it uniquely as
\(
\sum_{J \in \N^d} \hat{f}(J) H_J(\vec x).
\)
For more details on the Gaussian space and Hermite Analysis (especially from
the theoretical computer science perspective), we refer the reader to
\cite{Don14}.  Most of the facts about Hermite polynomials that we use in this
work are well known properties and can be found, for example, in \cite{Sze67}.

We denote by $f^{[k]}(x)$ the degree $k$ part of the Hermite expansion of $f$,
$f^{[k]} (\vec x) = \sum_{|J| = k} \hat{f}(J)\cdot H_J(\vec x)$.
We say that a polynomial $q$ is harmonic of degree $k$ if it is
a linear combination of degree $k$ Hermite polynomials, that is $q$ can be
written as
$$
q(\vec x) = q^{[k]}(\vec x) = \sum_{J: |J| = k} c_J H_J(\vec x)
$$

For a single dimensional Hermite polynomial it holds
$H_m'(x) = \sqrt{m} H'_{m-1}(x)$.  Using this, we obtain that for a multivariate
Hermite polynomial $H_M(\vec x)$, where $M = (m_1,\ldots, m_d)$ it holds
\begin{equation}
  \label{eq:hermite_nabla}
  \nabla H_M(\vec x) = \sqrt{m_i} H_{M - E_i}(\vec x) \in \R^d,
\end{equation}
where $E_i = \vec e_i$ is the multi-index that has $1$ position $i$ and $0$
elsewhere.  From this fact and the orthogonality of Hermite polynomials
we obtain
\begin{equation}
  \label{eq:hermite_nabla_dot}
  \E_{\vec x \sim \normal^d}[ \dotp{\nabla H_M(\vec x)}{\nabla H_L(\vec x)}]
  = |M| \delta_{M, L}.
\end{equation}


\begin{fact}\label{fct:harmonic_nabla_dot}
  Let $p, q$ be a harmonic polynomials of degree $k$.  Then
  $$
  \E_{\vec x \sim \normal}\left[\dotp{\nabla^{\ell} p(\vec x)}{\nabla^{\ell} q(\vec x)}\right]
  =
  k(k-1)\ldots(k-\ell+1)
  \E_{\vec x \sim \normal}[p(\vec x) q(\vec x)].
  $$
  In particular,
  $$
  \dotp{\nabla^{k} p(\vec x)}{\nabla^{k} q(\vec x)}
  =
  k!  \E_{\vec x \sim \normal}[p(\vec x) q(\vec x)] \;.
  $$
\end{fact}
\begin{proof}
  Write $p(\vec x) = \sum_{M: |M| = k} b_M H_M(\vec x)$ and
  $q(\vec x) = \sum_{M: |M| = k} c_M H_M(\vec x)$.
  Since the Hermite polynomials are orthonormal we obtain $\E_{\vec x \sim
  \normal}[p(\vec x) q(\vec x)] = \sum_{M: |M| = k} c_M b_M$.
  Now, using Equation~\eqref{eq:hermite_nabla} iteratively we obtain
  $$
  \E_{\vec x \sim \normal}
  \left[\dotp{\nabla^{\ell} H_M(\vec x)}{\nabla^{\ell} H_L(\vec x)}\right]
  = k (k-1)\ldots (k-\ell + 1) \delta_{M, L}.
  $$
  Using this equality, we obtain
  \begin{align*}
    \E_{\vec x \sim \normal}\left[\dotp{\nabla^{\ell} p(\vec x)}{\nabla^{\ell} q(\vec x)}\right]
  &=
  \E_{\vec x \sim \normal}\left[\dotp{\sum_M b_M \nabla^{\ell} H_M(\vec x)}{ \sum_L c_L \nabla^{\ell} H_L(\vec x) }\right]
  \\
  &=
  \sum_{M,L} b_M c_L \E_{\vec x \sim \normal}\left[\dotp{\nabla^{\ell} H_M(\vec x)}{\nabla^{\ell} H_L(\vec x)}\right]
  \\
  &= \sum_{M,L} b_M c_L  k (k-1)\ldots (k-\ell + 1) \delta_{M, L}.
  \\
  &= k (k-1)\ldots (k-\ell + 1) \E_{\vec x \sim \normal}[p(\vec x) q(\vec x)].
  \end{align*}
\end{proof}
Observe that for every harmonic polynomial $p(x)$ of degree $k$ we have that
$\nabla^{k} p(\vec x)$ is a symmetric tensor of order $k$.  Since the degree of
the polynomial is $k$ and we differentiate $k$ times this tensor no longer
depends on $\vec x$.  Using Fact~\ref{fct:harmonic_nabla_dot} we observe that
this operation (modulo a division by $\sqrt{k!}$) preserves the $L_2$ norm of
the harmonic polynomial $p$, that is $\E_{\vec x \sim \normal^d}[p^2(\vec x)] =
\snorm{2}{\nabla^{k} p(\vec x)}^2/k!$.

\begin{lemma}\label{lem:bound_coleration_app}
	Let $p(\vec x): \R^2 \mapsto \R$ be a function  and let $\matr U, \matr V \in
	\R^{2 \times d}$ be linear maps such that $\matr U \matr U^T = \matr V \matr
	V^T = \matr I \in \R^{2 \times 2}$.  Then,
	$
	\E_{\vec x \sim \normal^d}[p(\matr U \vec x) p(\matr V \vec x)] \leq
	\sum_{m=0}^{\infty} \snorm{2}{\nnew{\matr U \matr V^T}}^m \E_{\vec x \sim \normal^d}[(p^{[m]}(\vec x))^2] .
	$
\end{lemma}
\begin{proof}
	To simplify notation, write $f(\vec x) = p(\matr U \vec x)$ and $g(\vec x) =
	p(\matr V \vec x)$.  The (total) degree of $f$ is the same as the
	degree of $p$.   Write $f(\vec x) = \sum_{m=0}^{\infty} f^{[m]}(\vec x)$
	and $g(\vec x) = \sum_{m=0}^{\infty} g^{[m]}(\vec x)$.  Then using
	Fact~\ref{fct:harmonic_nabla_dot} we obtain
	\begin{align}
	\label{eq:correlated_inner_product}
	\E_{\vec x \sim \normal^d}[f(\vec x) g(\vec x)]
	&= \sum_{m=0}^\infty \E_{\vec x \sim \normal^d}[f^{[m]}(\vec x) g^{[m]}(\vec x)]
	= \sum_{m=0}^\infty \frac{1}{m!} \dotp{\nabla^m f^{[m]}(\vec x)}{\nabla^m g^{[m]}(\vec x)} \nonumber \\
	&
	= \sum_{m=0}^\infty \frac{1}{m!} \dotp{\nabla^m p^{[m]} (\matr U \vec x )}{\nabla^m p^{[m]}(\matr V \vec x)}\, .
	\end{align}
Denote by $\mathcal{U}\subseteq \R^d$ the image of the linear map $\matr U\nnew{^T}$.
	Now observe that, using the chain rule, for any function $h(\matr U \vec x):
	\vec \R^d \mapsto \R$ it holds
	$
	\nabla h(\matr U \vec x)
	= \partial_{i} h(\matr U \vec x) \matr U_{ij}\ \in \mathcal{U}\, ,
	$
	where we used Einstein's summation notation for repeated indices.
	Applying the above rule $m$-times we have that
	$$
	\nabla h(\matr U \vec x)
	= \partial_{i_m} \ldots \partial_{i_1} h(\matr U \vec x) \matr U_{i_1j_1}
	\ldots \matr U_{i_mj_m} \ \in \mathcal{U}^{\otimes m}\, .
	$$
	\nnew{
		Now we denote $\matr R = \nabla^m p^{[m]}(\bx)$ and observe that this
		tensor does not depend on $\bx$.  Moreover, denote $\matr M = \matr U \matr
		V^T $, $\matr S = \nabla^m p^{[m]} (\matr U \vec x) = (\matr U^T)^{\otimes m}
		\matr R \in \mathcal{U}^{\otimes m}$, and $\matr T = \nabla^m
		p^{[m]} (\matr V \vec x) = (\matr V^T)^{\otimes m} \matr R
		\in
		\mathcal{V}^{\otimes m}$.
		We have
		\begin{align*}
		\dotp{\matr S}{\matr T}
		= \dotp{(\matr U^T)^{\otimes m} \matr R}
		{(\matr V^T)^{\otimes m} \matr R}
		= \dotp{\matr R}{\matr M^{\otimes m} \matr R}
		\leq \snorm{2}{\matr M^{\otimes m}} \snorm{2}{\matr R}^2
		= m! \snorm{2}{\matr M}^m \E_{\vec x \sim \normal^d}[(p^{[m]}(\vec x))^2] \,,
		\end{align*}
	}
	where to get the last equality we used again Fact~\ref{fct:harmonic_nabla_dot}.
	To finish the proof we combine this inequality with
	Equation~\eqref{eq:correlated_inner_product}.
\end{proof}

In the following simple lemma we prove that random $2$-dimensional
subspaces in high dimensions are roughly orthogonal.

\begin{lemma} \label{lem:sq_rota_app}For any $0<c<1/2$, there exists a set $S$ of at least
  $2^{\Omega(d^c)}$ matrices in $\R^{2 \times d}$ such that for each pair
  $\matr A, \matr B \in S$, it holds
  $ \nnew{\snorm{2}{\matr A \matr B^T}} \leq O(d^{c-1/2})$.
\end{lemma}
\begin{proof}
  We are going to use the following lemma.
\begin{lemma}[Lemma 3.7 of \cite{DKS17-sq}]\label{lem:sq_paper}
	For any $0<c<1/2$, there is a set $S$ of at least $2^{\Omega(d^c)}$
	unit vectors in $\R^d$ such that for each pair of distinct $\vec u,\vec
	v\in S$, it hold $|\dotp{\vec u}{\vec v}| \leq O(d^{c-1/2})$.
\end{lemma}
  Let matrices $\matr A_1 ,\dots, \matr A_j$ in $\R^{2\times d}$, where
  $\matr A_i = \left (\vec u_{i,1}^T, \vec u_{i,2}^T \right)$, for some
  unit vectors $\vec u_{i,j}$ in $\R^d$. Then
  \begin{align*}
\nnew{\snorm{2}{\matr A_j \matr A_i^T} }=
    \snorm{2}{\matr A_j^T \matr A_i} =
    \sqrt{\sum_{x,y=1 }^2 (\vec u_{i,x}^T \vec
    u_{j,y})^2 }\leq 2 \max_{\vec u_{i,x},\vec u_{j,y}}| \cos\theta(\vec
    u_{i,x},\vec u_{j,y}) |\;.
  \end{align*}
  From Lemma~\ref{lem:sq_paper}, it holds that
  there exists a set of $2^{\Omega(d^c)}$ of unit vectors such that
  $|\cos\theta(\vec u,\vec v) |\leq O(d^{{c-1/2}})$, taking this vectors
  as columns in each matrix the result follows.
\end{proof}

\begin{lemma}\label{lem:function_low_ap}
 Let $f_{\sigma, \phi} \in \mcal{H}$.
For every polynomial $p(\x)$ of degree at most $k$, it holds $\E_{\x\sim \D}[f_{\sigma,\phi}(\x) \cdot p(\x)]=0$.
\end{lemma}
\begin{proof}
	Let $\vec w^{(m)}=(\cos \frac{2\pi m}{2k}, \sin  \frac{2\pi m}{2k})$ and $\alpha_m=(-1)^m$, for $m=1, \ldots, 2k$. Let $R_{\pi/k}$ be an operator over functions that rotates the coordinates by $\pi/k$ (i.e., $(x,y)\mapsto (x\cos \frac \pi k +y \sin \frac \pi k, -x\sin\frac \pi k +y\cos\frac \pi k)$). Then
	\begin{align}
	R_{\pi/k} [f](x,y)&= f\big( (x\cos \frac \pi k +y \sin \frac \pi k, -x\sin\frac \pi k +y\cos\frac \pi k)\big)
	\nonumber	\\&=\sigma\left(\sum_{m=1}^{2k-1} \alpha_m  \phi\left(\dotp{\vec x}{\vec w^{(m+1)}}\right) + \alpha_{2k} \phi\left(\dotp{\vec x}{\vec w^{(1)}}\right) \right) \nonumber\\&=
	\sigma\left(\sum_{m=1}^{2k} -\alpha_m  \phi\left(\dotp{\vec x}{\vec w^{(m)}}\right)\right)=-f(x,y)\;, \label{eq:lower1}
	\end{align}
	where to get the second equality we used that $ \alpha_i
	\phi\left(\dotp{(x\cos \frac \pi k +y \sin \frac \pi k, -x\sin\frac
		\pi k +y\cos\frac \pi k)}{\vec w^{(i)}}\right)=\alpha_i
	\phi\left(\dotp{(x,y)}{\vec w^{(i+1)}}\right) $ from basic
	trigonometric identities and in the last one we used that $\sigma$ is an
	odd function.
	Let $p(x,y)=(x+\iu y)^a(x-\iu y)^b$, where $\iu$ is the
	imaginary unit, then we are going to prove that $\E_{\bx \sim
		D}[f(\x) p(\x)]=0$ as long as $a-b\not\equiv k \mod 2k$. We have
	\begin{align}
	R_{\pi/k}[p](x,y)&=R_{\pi/k}[ (x+\iu y)^a (x-\iu y)^b]=R_{\pi/k} [(x^2+y^2)^{a+b}  e^{-\iu \theta (a-b)}]\nonumber\\&=
	(x^2+y^2)^{a+b}  e^{\iu (\theta+\pi/k) (a-b)  } =  e^{\iu (\pi/k) (a-b)  } p(x,y)\;,\label{eq:lower2}
	\end{align}
	where $\theta$ is the argument (or the ``phase")  of $x+\iu y$.
	This means that $p(x,y)$ is an eigenfunction of $R_{\pi/k}$ and $ e^{\iu (\pi/k) (a-b)  }$
	the corresponding eigenvalue. Thus, it holds
	$$e^{\iu (\pi/k) (a-b)  }  \E_{\bx \sim \D}[ f(\x) p(\x)] =\E_{\bx \sim \D}[ f(\x) R_{\pi/k}[p](\x)]  =  
	\E_{\bx \sim \D}[R_{-\pi/k} [f](\x) p(\x)]  =-\E_{\bx \sim \D}[f(\x) p(\x)]\;, $$
	where we used that $R_{\pi/k}$ is an adjoint operator in the inner product space of continuous functions 
	along with Equations \eqref{eq:lower1}, \eqref{eq:lower2}. Thus, $\E_{\bx \sim \D}[f(\x) p(\x)]=0$, 
	when  $ e^{\iu (\pi/k) (a-b)  }  \neq-1$, which happens when $a-b\not\equiv k \mod 2k$. 
	To conclude the proof, note that every polynomial at most degree $k$ is a linear combination 
	of the polynomials $p(x,y)=(x+\iu y)^a(x-\iu y)^b$ where $a,b\leq k$. This can be seen by setting 
	$x=\frac{z+\bar{z}}{2}$ and $y=\frac{z-\bar{z}}{2\iu}$,
	where $z=x+\iu y$ and $\bar{z}=x-\iu y$.
\end{proof}
\subsection{Interpretation of the class $\mcal H$ } \label{sec:iterpolation}
In order for the lower bound construction of Section~\ref{sec:sq}  to produce
useful lower bounds, it will be necessary that the function given in
Lemma~\ref{lem:function_low} is non-vanishing. It turns out that this is the
case under fairly weak conditions. In order to state our final result, we
will first introduce some terminology:
\begin{definition}
	For an integer $k$ the $k$-parity-part of a function
	$\phi:\R\rightarrow \R$ is the odd part of $\phi$ if $k$ is odd and
	the even part of $\phi$ if $k$ is even.
\end{definition}
\begin{definition}
	For functions $f$ on $\R^2$ define the operators $R_k$ to be the
	rotation by $\pi/k$ and define $S_k(f) = \sum_{s=1}^{2k} (-1)^s
	R_k^s(f).$
\end{definition}

Given these we have the following result implying that $S_k \phi(x) \neq 0$
for a number of functions of interest. For example, if $\phi(x) = \max(0,x)$,
is a ReLU, then the even part of $\phi$ is the absolute value function, so
$S_k \phi \neq 0$ for any even $k$. Similarly, if $\phi$ is a sigmoid, $S_k
\phi \neq 0$ for any odd $k$.
\begin{proposition}
	Let $\phi:\R\rightarrow\R$ be a function with $\E[\phi^2(x)]<\infty$.
	Then $S_k \phi(x) = 0$ if and only if the $k$-parity-part of $\phi$ is
	a polynomial of degree less than $k$.
\end{proposition}
\begin{proof}
	We begin by noting that $S_k \phi(x) \neq 0$ if and only if $(S_k
	\phi(x))^{[m]} \neq 0$ for some $m$. We note that as a rotation $R_k$
	preserves the degree-$m$ Hermite parts of a function, and therefore so
	does $S_k$. In particular, $(S_k \phi(x))^{[m]} = (S_k
	\phi(x)^{[m]}).$ In order to analyze this, we consider the variables
	$z=x+iy$ and $\bar{z} = x-iy$. We note that if $\phi(x)$ in one
	variable is given by the Hermite expansion $\phi(x) =
	\sum_{t=0}^\infty a_t h_t(x)$, that the two-variable version is given
	by $\sum_{t=0}^\infty a_t h_m((z+\bar{z})/2).$ Furthermore, we have
	that $(\phi(x))^{[m]} = a_m h_m((z+\bar{z})/2)$.

	Now if $a_m=0$, then $(\phi(x))^{[m]} = 0$ and therefore
	$S_k(\phi(x))^{[m]} = 0$. Otherwise, $a_m h_m(x)$ has non-vanishing
	$x^t$ coefficients for all $t\leq m$ with $t\equiv m\pmod{2}$.
	Therefore, in this case $(\phi(x))^{[m]}$ will have a non-vanishing
	$z^a \bar{z}^b$ coefficient for all $a,b\geq 0$ with $a+b\leq m$ and
	$a+b\equiv m \pmod{2}$. Next, we need to understand what $S_k$ does to
	$z^a \bar{z}^b$.

	For this we note that $R z = e^{\pi i /k}z$ and $R \bar{z} =
	e^{-\pi i/k}\bar{z}$. Thus $R (z^a \bar{z}^b) = e^{\pi i (a-b)/k}z^a
	\bar{z}^b.$ Therefore,
	$$
	S_k (z^a \bar{z}^b) = z^a \bar{z}^b \sum_{t=1}^{2k} e^{2 \pi i
		(a-b+k)/(2k)} = \begin{cases} 2k z^a \bar{z}^b & \textrm{if }a-b\equiv
	k \pmod{2k} \\ 0 & \textrm{else} \end{cases}
	$$
	Thus, $S_k (\phi(x))^{[m]}$ will be non-vanishing if and only if $a_m
	\neq 0$ and there are some $a,b\geq 0$ with $a+b\leq m, a+b\equiv
	m\pmod{2}$ and $a-b\equiv k\pmod{2k}$. We claim that such $a,b$ exist
	if and only if $m\equiv k \pmod{2}$ and $m\geq k$. The only if part of
	this condition is clear. For the if part, we note that if these
	conditions are satisfied, we may take $a=\frac{m+k}{2}$ and
	$b=\frac{m-k}{2}$.

	Therefore, we have that $S_k \phi(x) \neq 0$ if and only if there is
	some $m\equiv k\pmod{2}$ with $a_m\neq 0$ and $m\geq k$. Note that the
	$k$-parity-part of $\phi$ has the same Hermite coefficients as $\phi$
	for $m\equiv k\pmod{2}$ and 0 coefficient for $m\not\equiv k\pmod{2}$.
	Thus, $\phi$ has a non-vanishing coefficient for some $m\geq k,
	m\equiv k \pmod{2}$ if and only if the $k$-parity-part of $\phi$ has
	some non-vanishing coefficient of degree $m\geq k$. Of course this
	happens if and only if the $k$-parity-part of $\phi$ is not a
	polynomial with degree less than $k$.
	This completes our proof.
\end{proof}
 \bibliographystyle{alpha}

\end{document}